\documentclass[lettersize,journal]{IEEEtran}
\usepackage{amsmath,amsfonts}
\usepackage{algorithmic}
\usepackage{algorithm}
\usepackage{array}
\usepackage[caption=false,font=normalsize,labelfont=sf,textfont=sf]{subfig}
\usepackage{textcomp}
\usepackage{stfloats}
\usepackage{url}
\usepackage{verbatim}
\usepackage{graphicx}
\usepackage{cite}
\newcounter{case}
\renewcommand{\thecase}{\arabic{case}}
\usepackage{threeparttable}
\usepackage{amsmath}
\usepackage{amssymb}
\usepackage{mathtools}
\usepackage{amsthm}
\usepackage{tcolorbox}
\usepackage{booktabs}
\usepackage{pifont} % 提供 \ding{51} 和 \ding{55}
\usepackage{xcolor} % 颜色支持
\usepackage{float}

\usepackage{makecell}
\usepackage{hyperref} 
\usepackage[capitalize,noabbrev]{cleveref}

\usepackage{booktabs, tabularx, adjustbox, ltablex}
\usepackage{amsmath}
\usepackage{ulem}
\theoremstyle{plain}
\newtheorem{Theorem}{Theorem}
\newtheorem{Definition}[Theorem]{Definition}
\newtheorem{Corollary}[Theorem]{Corollary}

\newtheorem{Lemma}[Theorem]{Lemma}

\theoremstyle{remark}

\usepackage{multirow}
\theoremstyle{Case}

\usepackage{xcolor}
\usepackage{tcolorbox}
\definecolor{mblue}{HTML}{BDD8E4}
\newcommand{\statebullet}{\textcolor{mblue}{\bullet}}

\newtcolorbox{BoxedDef}{
    colback=white,                 % 背景色：纯白
    colframe=black!90!white,       % 边框色：90%黑（极深灰，比纯黑柔和）
    boxrule=0.5pt,                 % 边框粗细：极细
    sharp corners,                 % 强制直角（正方形框）
    equal height group=expertboxes % 等高编组（如果在同一排，盒子会等高）
    left=2pt,        % 边框到文字的左边距
    right=2pt,       % 边框到文字的右边距
    top=3pt,         % 边框到文字的上边距
    bottom=3pt,      % 边框到文字的下边距
}

\definecolor{mpurple}{HTML}{CF9AB3}
\newcommand{\statepullet}{\textcolor{mpurple}{\bullet}}

\definecolor{myred}{HTML}{86003E}
\newcommand{\staterellet}{\textcolor{myred}{\bullet}}

\begin{document}

\title{Weisfeiler--Leman Test on Combinatorial Complexes: A Unified Expressive Power Framework in Topological Neural Networks}

% A Unified Weisfeiler--Leman Framework for Higher-Order Relational Learning

\author{Jiawen Chen, Qi Shao, Zhiqiang  Ge,~\IEEEmembership{Senior Member, IEEE}, Duxin Chen, Wenwu Yu,~\IEEEmembership{Senior Member, IEEE}
\thanks{This research was supported by the National Science and Technology Major Project of the Ministry of Science and Technology of China (Grant No.2024ZD0608104), the National Natural Science Foundation of China (Grants No.62233004, 62273090, and 62073076), the Zhishan Youth Scholar
Program, the Jiangsu Provincial Scientific Research Center of Applied Mathematics (Grant No. BK20233002), the Natural Science Foundation of Jiangsu Province of China (Grants No. BK20253018, and BK20253020), the Open Research Project of the State Key Laboratory of Industrial Control Technology, China (Grant No. ICT2025B54). (Corresponding authors: Duxin Chen, Wenwu Yu.) }% <-this % stops a space
\thanks{Jiawen Chen, Qi Shao, Zhiqiang Ge, Duxin Chen, and Wenwu Yu are with the Jiangsu Key Laboratory of Networked Collective Intelligence, School of Mathematics, Southeast University, Nanjing 210096, China. (e-mail:chenjiawen@seu.edu.cn; shaoqi@seu.edu.cn; zhiqiang.ge@hotmail.com; chendx@seu.edu.cn; wwyu@seu.edu.cn).} 
}
% <-this % stops a space
% \thanks{This paper was produced by the IEEE Publication Technology Group. They are in Piscataway, NJ.}% <-this % stops a space
% \thanks{Manuscript received April 19, 2021; revised August 16, 2021.}
% }

% The paper headers
% \markboth{Journal of \LaTeX\ Class Files,~Vol.~14, No.~8, August~2021}%
% {Shell \MakeLowercase{\textit{et al.}}: A Sample Article Using IEEEtran.cls for IEEE Journals}

% \IEEEpubid{0000--0000/00\$00.00~\copyright~2021 IEEE}
% Remember, if you use this you must call \IEEEpubidadjcol in the second
% column for its text to clear the IEEEpubid mark.

\maketitle
 
\begin{abstract}
Topological neural networks have emerged as effective tools for modeling higher-order relational structures beyond pairwise graphs, including hypergraphs, simplicial complexes, and cell complexes. However, existing Weisfeiler--Leman type expressivity analyses are typically developed on different structural domains and rely on domain-specific neighborhood systems, making their expressive powers difficult to compare within a common formalism. In this paper, we introduce the Combinatorial Complex Weisfeiler--Leman (CCWL) framework, a unified expressive power refinement defined on combinatorial complexes. By exploiting the ability of combinatorial complexes to represent both set-type relations and part--whole hierarchies, CCWL performs topological color refinement through four structural neighborhoods: boundary, co-boundary, lower adjacency, and upper adjacency. We show that, under specified lifting maps, CCWL can simulate several domain-specific WL-type refinements, thereby providing a common theoretical baseline for analyzing topological message passing. We further study the neighborhood sufficiency problem and prove that, under explicit coverage conditions, a reduced refinement using only lower- and upper-adjacent bridge information preserves the distinguishing power of the full four-neighborhood CCWL refinement. Guided by this theoretical result, we instantiate the reduced refinement as the Combinatorial Complex Isomorphism Network (CCIN). Experiments on synthetic and real-world benchmarks demonstrate that CCIN achieves competitive performance against representative graph and topological neural network baselines. Ablation studies and resource-efficiency analyses further support the effectiveness of the proposed lower/upper-neighborhood design. The source code and datasets are publicly available at \url{https://github.com/jiawenchen10/CCIN}.
\end{abstract}


\section{Introduction}

\IEEEPARstart{R}{elational} data in the real world from biological networks to physical simulations, naturally exhibits complex, multi-way interactions that extend far beyond simple pairwise graphs. To capture these higher-order dependencies, the field of topological neural networks ~\cite{pham2025topological} has rapidly evolved. Various topological structures, such as hypergraphs~\cite{feng2024hypergraph,zhangimproved,besta2024demystifying}, simplicial complexes~\cite{wu2023simplicial,gurugubelli2023sann,tahademystifying}, and cellular complexes~\cite{bodnar2021weisfeiler,eitantopological}, have been widely adopted to represent set-based and part-whole hierarchical relations~\cite{pmlr-v235-papamarkou24a,millan2025topology,liu2023wl}.
However, this rapid expansion has led to a fragmented landscape. Existing higher-order neural networks are often designed for specific structural domains~\cite{huang2021unignn,huang2024higher,battiloro2024latent}, relying on different neighborhood definitions and customized message-passing mechanisms. Their expressive powers are difficult to compare under a unified theoretical framework~\cite{pmlr-v235-papamarkou24a}.

% \begin{figure}[htbp]
% \centering
% \setlength{\belowcaptionskip}{-15pt}
% \includegraphics[width=1\linewidth]{figures001.pdf}
% \caption{Illustration of Weisfeiler--Leman tests on four neighborhood function.}
% % \label{figs001}
% \end{figure}

The Weisfeiler--Leman (WL) test serves as a standard tool for characterizing the expressive power of Graph Neural Networks (GNNs)~\cite{weisfeiler1968reduction,xupowerful,feng2022powerful}. Recent studies have extended WL-type refinements beyond graphs, leading to hypergraph WL~\cite{zhangimproved,feng2024hypergraph}, simplicial WL-type tests induced by message-passing simplicial networks~\cite{bodnar2021weisfeilernips}, and cellular WL~\cite{bodnar2021weisfeiler}. However, these variants are defined on different structural domains and rely on different neighborhood systems, making their expressive powers difficult to compare under a common formalism. Consequently, the TDL community still lacks a unified theoretical baseline for evaluating expressivity across heterogeneous topological domains~\cite{papp2022theoretical}. Combinatorial complexes (CCs) have recently emerged as a unifying topological formalism that subsumes both set-type relations, such as graphs and hypergraphs, and part-whole hierarchies, such as simplicial and cell complexes~\cite{hajij2022topological,hajij2023combinatorial}. This makes CCs a natural mathematical domain for developing a unified WL-type refinement for topological message passing. Nevertheless, defining such a refinement on CCs is nontrivial. Despite the structural unification provided by CCs, the topological deep learning community still lacks a generalized theoretical baseline built upon this framework~\cite{papp2022theoretical,pmlr-v235-papamarkou24a}. Specifically, a universal WL-type test on CCs, one capable of systematically evaluating, comparing, and aligning expressivity across all heterogeneous topological domains—remains notably absent. 

Recent studies have explored combinatorial complexes as a common domain for higher-order learning and expressivity analysis. For instance, TopoTune~\cite{papillontopotune} follows the CWN-style topological message-passing paradigm~\cite{bodnar2021weisfeiler} by mapping higher-order relations to a strict Hasse graph, where propagation is performed along incidence-induced edges. This provides an architectural mechanism for learning on combinatorial complexes, but the strict Hasse-graph representation reduces the native neighborhood structure of CCs to incidence-driven propagation.  Logical expressiveness~\cite{akbari2026the} studies WL-type refinements on combinatorial complexes from a hierarchical perspective. CCMamba~\cite{chen2026ccmamba} addresses the scalability and long-range dependency issues of topological message passing by introducing state-space modeling to avoid the quadratic complexity of attention-based propagation, but it is not designed to characterize WL-type neighborhood sufficiency. A combinatorial complex naturally induces four structural relations: boundary, co-boundary, lower adjacency, and upper adjacency~\cite{eitantopological}. A full CCWL update would aggregate color multisets from all four relations. Nevertheless, these architectures do not provide a unified WL-type expressivity theory for combinatorial complexes, nor do they characterize which neighborhood relations are sufficient for preserving the distinguishing power of a full combinatorial-complex refinement.

To bridge these theoretical and architectural divides, we introduce the Combinatorial Complex Weisfeiler--Leman (CCWL) framework. By providing a unified expressivity baseline for topological message passing, our work addresses key challenges in topological deep learning recently highlighted in the literature~\cite{pmlr-v235-papamarkou24a}, including cross-domain message-passing generalization and expressivity analysis for topological neural networks. Our main contributions are summarized as follows:

\begin{itemize}
\item \textbf{Unified formalization.}
We introduce CCWL, a unified Weisfeiler--Leman refinement framework defined on combinatorial complexes. By formulating color refinement through boundary, co-boundary, lower-adjacency, and upper-adjacency relations, CCWL provides a common expressivity formalism for analyzing graphs, hypergraphs, simplicial complexes, and cell complexes within a single topological domain.

\item \textbf{Theoretical expressivity.}
We formally define the CCWL refinement and study the neighborhood sufficiency problem on combinatorial complexes. We prove that, under explicit coverage conditions, the refinement restricted to lower and upper adjacencies preserves the distinguishing power of the full four-neighborhood CCWL refinement. We further establish expressivity relations between CCWL and domain-specific WL variants, showing that several existing higher-order refinement procedures can be simulated by CCWL under appropriate lifting maps.

\item \textbf{Empirical validation.}
Guided by the theoretical results, we instantiate CCWL as the Combinatorial
Complex Isomorphism Network (CCIN). Extensive experiments on synthetic and
real-world benchmarks demonstrate that CCIN achieves competitive performance
against representative graph and topological neural network baselines.
\end{itemize}

\section{Related Work}

In this section, we review prior studies on topological neural networks for higher-order structures and expressive-power analysis of topological message passing. We further discuss recent combinatorial-complex frameworks and clarify the difference of their expressivity-oriented approaches.

\subsection{Topological Neural Networks for Higher-Order Structures}

Graph neural networks have achieved remarkable success in modeling pairwise relational data. However, many real-world systems involve interactions beyond pairs, including group interactions, simplicial dependencies, cellular structures, and part-whole relations. To capture such dependencies, topological neural networks have been developed for higher-order relational learning in topological deep learning~\cite{pham2025topological,verma2024topological,chen2021topological}.
Hypergraph neural networks model high-order relations through node--hyperedge incidences~\cite{feng2019hypergraph,huang2021unignn,xie2025k}. Representative models, such as HGNN+~\cite{9795251}, UniGNN~\cite{huang2021unignn}, and recent high-order hypergraph architectures, improve message passing by explicitly modeling vertex--hyperedge interactions. 
\begin{table}[htbp]
\scriptsize 
\centering
\caption{Overview of expressive-power methods and their corresponding structural domains.}
\label{tab:expressivity_overview}
\renewcommand{\arraystretch}{1.3} 
\begin{tabular}{lp{1.2cm} p{5.0cm}}
\toprule  
\textbf{Category} & \textbf{Subtype} & \textbf{Representative methods} \\
\midrule

\multirow{2}{*}{\makecell[l]{Set-Type \\ Relational \\ Expressivity}} 
& Graph 
& GIN~\cite{xupowerful}, DGCNN~\cite{zhang2018end}, IGN~\cite{cai2022convergence}, PPGNs~\cite{maroninvariant}, GSN~\cite{bouritsas2022improving}, rMPNN~\cite{paolino2024weisfeiler}, HomoGNN~\cite{zhangbeyond}  \\
\cmidrule{2-3}

& Hypergraph 
& Hypergraph WL~\cite{feng2024hypergraph}, KGWL~\cite{zhangimproved}, ED-HNN~\cite{wangequivariant}, RePHINE~\cite{immonen2023going} \\ 

\midrule

\multirow{2}{*}{\makecell[l]{Part-Whole \\ Relational \\ Expressivity}}  
& Simplicial 
& MPSN~\cite{bodnar2021weisfeilernips}, IMPSN~\cite{eijkelboom2023n}, EMPSN~\cite{eijkelboom2023n}, ETNNs~\cite{battiloro2024n}, SMCN~\cite{eitantopological} \\ 
\cmidrule{2-3}

& Cell complex 
& CWN~\cite{bodnar2021weisfeiler}, TopNets~\cite{verma2024topological}, BSCNets~\cite{chen2022bscnets}, PathCWL~\cite{truong2024weisfeiler} \\

\midrule

\makecell[l]{Topological \\ Expressivity} 
& \makecell[l]{Combinatorial \\ complex} 
& TopoTune~\cite{papillontopotune}, kCCWL~\cite{akbari2026the}, CCMamba~\cite{chen2026ccmamba}, \textbf{CCWL and CCIN (ours)} \\

\bottomrule
\end{tabular}
\end{table}
While effective for set-valued relations, they mainly capture incidence structures and do not fully model richer cross-rank topological relations, such as boundary, co-boundary, lower adjacency, and upper adjacency.
Simplicial neural networks use boundary and co-boundary operators to propagate information among vertices, edges, triangles, and higher-dimensional simplices~\cite{bodnar2021weisfeilernips,eijkelboom2023n,wu2023simplicial,liuclifford}. Cell complex neural networks further generalize message passing to regular cell complexes~\cite{hajij2020cell,bodnar2021weisfeiler}. Recent studies incorporate attention, equivariance, and geometric constraints to enhance topological architectures~\cite{battiloro2024n,ballester2024attending}. These models are typically tied to fixed topological domains and fixed subsets of neighborhood relations, making their expressive power difficult to compare under a unified theoretical framework~\cite{tahademystifying,eitantopological}.

\subsection{Expressive Power of Topological Neural Networks}
The Weisfeiler--Leman test has been widely used to analyze the expressive power of neural message passing. For graphs, the 1-WL test provides a theoretical reference for standard message-passing GNNs~\cite{weisfeiler1968reduction,maroninvariant,xupowerful,feng2022powerful,puny2023equivariant}. Higher-order variants, such as $k$-WL, sparse WL, and loopy WL, improve discriminative power but usually incur higher computational complexity~\cite{morris2020weisfeiler,paolino2024weisfeiler,papp2022theoretical,10818675}. In parallel, substructure-aware and equivariant GNNs enhance expressivity by incorporating subgraph, path, and isomorphism-based information~\cite{zhangbeyond,zhangrethinking,thiede2021autobahn,zhang2023complete,zhang2024schur,bevilacquaequivariant,bouritsas2022improving,zhang2024expressive}. 
WL-type refinements have also been extended to higher-order topological domains. Hypergraph WL tests refine node and hyperedge representations through bipartite incidence relations~\cite{zhangimproved,feng2024hypergraph}. Simplicial WL-type refinements are induced by boundary and co-boundary operators~\cite{bodnar2021weisfeilernips,eitantopological}, while cellular WL refines cells of different dimensions~\cite{bodnar2021weisfeiler,chen2022bscnets,truong2024weisfeiler}. However, these variants are defined on different structural domains and rely on different neighborhood systems, making their expressive powers difficult to compare under a common formalism~\cite{pmlr-v235-papamarkou24a}. More importantly, they do not identify which neighborhood relations are sufficient when these structures are represented within a unified combinatorial-complex framework. This motivates a WL-type refinement on combinatorial complexes that can unify heterogeneous topological domains while characterizing the role of different neighborhood relations.

Combinatorial complexes provide a unified formalism for set-type relations and part-whole hierarchies~\cite{hajij2022topological,hajij2023combinatorial}. Existing works have recently been studied from different aspects. TopoTune~\cite{papillontopotune} extends the CWN-style message-passing paradigm~\cite{bodnar2021weisfeiler} by mapping higher-order relations to strictly Hasse graph and employing WL test on incidence-based propagation, while other works~\cite{chen2026ccmamba,akbari2026the} focuses on hierarchical WL-type refinements on combinatorial complexes. As shown in Table~\ref{tab:expressivity_overview}, these methods do not aim to unify WL variants across graphs, hypergraphs, simplicial complexes, and cell complexes. In contrast, our work defines a WL-type refinement directly on combinatorial complexes, establishes its relation to domain-specific WL variants, and characterizes when a reduced lower/upper-neighborhood rule preserves the distinguishing power of full CCWL.

\section{Preliminaries}
In this section, we first introduce the basic notions of combinatorial complexes and then define a Weisfeiler--Leman test refinement on combinatorial complexes. The main notations are summarized in Table~\ref{tab:notations}.

\subsection{Combinatorial Complexes}
Combinatorial complexes provide a general topological representation~\cite{hajij2023combinatorial,wu2023simplicial} that
consists of both set-type structures, such as graphs and hypergraphs, and
part-whole structures, such as simplicial and cell complexes. Formally, a
combinatorial complex is defined as follows.

\begin{Definition}[\textbf{Combinatorial Complex}]
\label{definition01}
A combinatorial complex (CC) is a triple
$\mathcal{CC}=(\mathrm{S},\mathrm{C},\mathrm{rk})$, where $\mathrm{S}$ is a
set of base elements, $\mathrm{C}\subseteq \mathcal{P}(\mathrm{S})$ is a set
of cells, and $\mathrm{rk}:\mathrm{C}\to \mathbb{Z}_{\ge 0}$ is a rank
function. For every $v\in\mathrm{S}$, the singleton $\{v\}$ belongs to
$\mathrm{C}$ and satisfies $\mathrm{rk}(\{v\})=0$. The rank function is
order-preserving: for any $\sigma,\tau\in\mathrm{C}$, if
$\sigma\subseteq\tau$, then $
\mathrm{rk}(\sigma)\leq \mathrm{rk}(\tau)$. 
A cell $\sigma\in\mathrm{C}$ with $\mathrm{rk}(\sigma)=k$ is called a
$k$-cell. The dimension of $\mathcal{CC}$ is defined as $
\dim(\mathcal{CC}) := \max_{\sigma\in\mathrm{C}} \mathrm{rk}(\sigma).$
\end{Definition}

The rank function allows us to refer to different structural levels of a
complex. For example, the $0$-skeleton contains only $0$-cells, the $1$-skeleton
contains $0$- and $1$-cells, and the $2$-skeleton further includes $2$-cells,
such as faces. The incidence structure of $\mathcal{CC}$ is described by a boundary relation.
We write $\eta\prec\sigma$ if $\eta$ is an immediate boundary cell of $\sigma$.
Its reflexive-transitive closure is denoted by $\preceq$. Thus,
$\eta\preceq\sigma$ means that $\eta$ lies on the boundary chain of $\sigma$.
Throughout this work, we focus on rank-adjacent incidences, i.e., incidences
between cells whose ranks differ by one. Non-rank-adjacent incidences can be
obtained through chains of rank-adjacent incidences.

\begin{table}[htbp]
    \centering
    \caption{Notation table for the main symbols.}
    \renewcommand{\arraystretch}{1.2}  
    \begin{tabular}{c|l}
    \toprule
     \textbf{Symbol} & \textbf{Description} \\
     \midrule
     $\mathcal{CC}$ & A combinatorial complex $(\mathrm{S},\mathcal{C},\mathrm{rk})$ \\ 
     $\mathrm{S}$ & Set of base elements, e.g., nodes or $0$-cells \\ 
     $\mathrm{C}$ & Set of cells, $\mathrm{C}\subseteq \mathcal{P}(\mathrm{S})$ \\ 
     $\sigma,\tau,\eta,\gamma$ & Cells in $\mathcal{C}$ \\ 
     $\mathrm{rk}(\sigma)$ & Rank of cell $\sigma$ \\
     $\prec$ & Immediate boundary relation between cells \\
     $\preceq$ & Reflexive-transitive closure of $\prec$ \\
     $\mathcal{B}(\sigma)$ & Boundary neighborhood of cell $\sigma$ \\
     $\mathcal{C}(\sigma)$ & Co-boundary neighborhood of cell $\sigma$ \\
     $\mathcal{N}_{\downarrow}(\sigma)$ & Lower-adjacent neighborhood of cell $\sigma$ \\
     $\mathcal{N}_{\uparrow}(\sigma)$ & Upper-adjacent neighborhood of cell $\sigma$ \\
     $\{\{\cdot\}\}$ & Multiset operation \\
     $c^{(t)}(\sigma)$ & Color assigned to cell $\sigma$ at iteration $t$ \\
     $c_{\mathcal{N}}^{(t)}(\sigma)$ & Color multiset from neighborhood $\mathcal{N}$ at iteration $t$ \\
     $\mathrm{HASH}(\cdot)$ & Injective multiset hashing function \\
     $h_\sigma^{(l)}$ & Feature embedding of cell $\sigma$ at layer $l$ \\
     $m_{\mathcal{N}}^{(l+1)}$ & Aggregated message from neighborhood $\mathcal{N}$ at layer $l+1$ \\
     $\mathrm{AGG}(\cdot)$ & Permutation-invariant aggregation function \\
     $\phi^{(l)}, \psi_{\mathcal{N}}^{(l)}$ & Learnable update and message functions at layer $l$ \\
    \bottomrule
    \end{tabular}
    \label{tab:notations}
\end{table}

\subsection{Combinatorial Complex Weisfeiler--Leman}

We next define the four neighborhoods in CCs.

\begin{Definition}[\textbf{Neighborhood Functions}]
\label{cellneighhood}
Let $\mathcal{CC}=(\mathrm{S},\mathrm{C},\mathrm{rk})$ be a combinatorial
complex and let $\sigma\in\mathrm{C}$. The four structural neighborhoods of
$\sigma$ are defined as follows:

1. Boundary neighborhood: $\mathcal{B}(\sigma) = \{ \tau| \tau \prec \sigma \}$.
These are the lower-rank cells contained in the boundary of
$\sigma$.
% , with the neighborhood specified by the boundary matrix $\mathcal{B}_r$.
% , i.e. the neighborhood cell vectors connected with edges $\tau$. 

2. Co-boundary neighborhood: $\mathcal{C}(\sigma) = \{ \tau| \sigma \prec \tau \}$. 
These are the higher-rank cells that contain $\sigma$ in their boundary.
% specified by the boundary matrix $\mathcal{B}^{T}_r$. 
% For instance, the co-boundary cells of a node are the edges it is part of. 

3. Lower neighborhood: $\mathcal{N}_{\downarrow}(\sigma)
=
\{\tau\in\mathrm{C}\mid
\tau\neq\sigma,\ 
\mathrm{rk}(\tau)=\mathrm{rk}(\sigma),\
\mathcal{B}(\sigma)\cap\mathcal{B}(\tau)\neq\emptyset
\}$. 
These are cells of the same rank as $\sigma$ that share at least one boundary cell with $\sigma$.

4. Upper neighborhood: $\mathcal{N}_{\uparrow}(\sigma)
=
\{
\tau\in\mathrm{C} $ $ | 
\tau\neq\sigma,\ 
\mathrm{rk}(\tau)=\mathrm{rk}(\sigma),\
\mathcal{C}(\sigma)\cap\mathcal{C}(\tau)\neq\emptyset
\}.
$
These are cells of the same rank as $\sigma$ that share at least one co-boundary cell with $\sigma$.

\end{Definition}
A message may propagate across ranks, for example, from a $0$-cell to an incident $1$-cell. 

Based on these four structural neighborhoods, we define a
Weisfeiler--Leman-type color refinement on combinatorial complexes. Let
$c^{(t)}:\mathrm{C}\to\mathbb{N}$ denote the coloring of cells at iteration
$t$. The multiset operator is denoted by $\{\{\cdot\}\}$.

\begin{Definition}[\textbf{Neighborhood Color Multisets}]
\label{def:neighborhood_color_multisets}
Let $c:\mathrm{C}\to\mathbb{N}$ be a coloring of the cells in a combinatorial
complex $\mathcal{CC}=(\mathrm{S},\mathrm{C},\mathrm{rk})$. For two cells
$\sigma,\tau\in\mathrm{C}$, define the shared boundary and shared co-boundary
bridge sets as $\mathcal{B}(\sigma, \tau):=\mathcal{B}(\sigma)\cap \mathcal{B}(\tau)$, and $\mathcal{C}(\sigma,\tau):=\mathcal{C}(\sigma)\cap\mathcal{C}(\tau)$. For cell $\sigma\in\mathrm{C}$, we define four neighborhood color
multisets:

1. The boundary colors: $c_{\mathcal{B}}(\sigma) = \{\{c_{\tau} | \tau\in \mathcal{B}(\sigma)\}\}$.

2. The co-boundary colors: $c_{\mathcal{C}}(\sigma) = \{\{c_{\tau} \mid \tau\in \mathcal{C}(\sigma)\}\}$. 

3. The lower adjacent colors: $c_{\mathcal{N}_{\downarrow}}(\sigma) = \{\{(c_{\tau}, c_{\delta}) | \tau \in \mathcal{N}_{\downarrow}(\sigma) \text{ and } \delta \in \mathcal{B}(\sigma, \tau)\}\}. $

4. The upper adjacent colors: $c_{\mathcal{N}_{\uparrow}}(\sigma) = \{\{(c_{\tau}, c_{\delta}) | \tau \in \mathcal{N}_{\uparrow}(\sigma) \text{ and } \delta \in \mathcal{C}(\sigma, \tau)\}\}.$
\end{Definition}

Building on the four neighborhood color multisets, we define the CCWL iterative color refinement algorithm.

\begin{figure}[htbp]
  \centering
\includegraphics[width=1\linewidth]{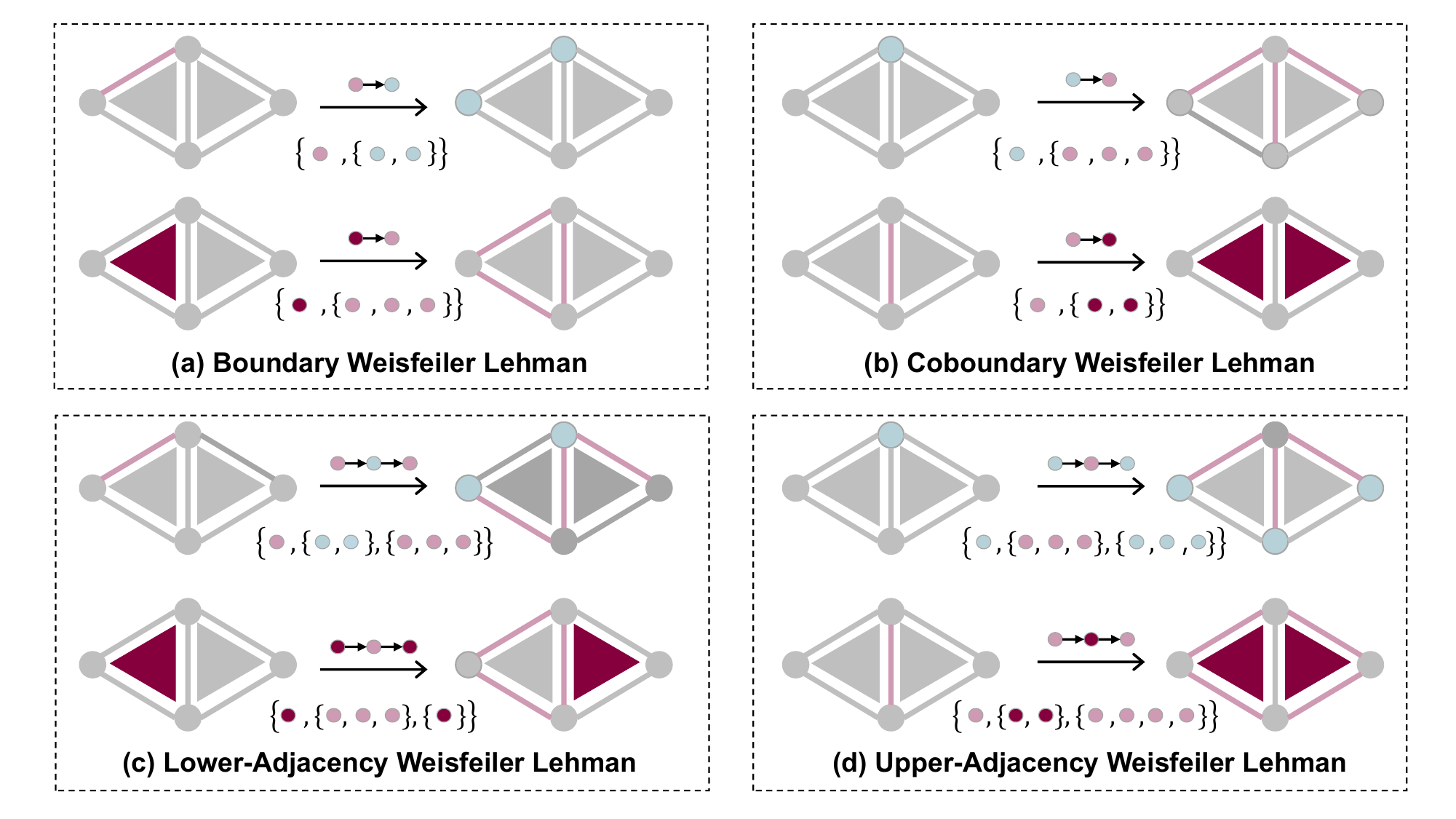}
  \caption{Illustration of the four structural neighborhood relations used in CCWL: boundary, co-boundary, lower adjacency, and upper adjacency.}
  \label{figs001}
\end{figure}

\begin{Definition}[\textbf{CCWL Test}]
\label{ccwldef}
Let $\mathcal{CC}_1 = (\mathrm{S}_1, \mathrm{C}_1, \mathrm{rk}_1)$ and $\mathcal{CC}_2 = (\mathrm{S}_2, \mathrm{C}_2, \mathrm{rk}_2)$ be two $\mathcal{CC}$, and a coloring function $c$ on cells.
Given an initial coloring $c^{(0)}:\mathrm C_1\cup\mathrm C_2\to\mathbb N$,
the CCWL iteration proceeds as follows.

1. All the cell of the same rank $\sigma\in \mathrm{C}_1, \tau\in \mathrm{C}_2$ are initialized with the same color, defined as $c^{(0)}(\sigma),c^{(0)}(\tau)$ .

2. Given the color $c^{(t)}$ of cell $\sigma$ at iteration $t$, we update the color of next iteration as $c^{(t+1)}(\sigma)$ by injective mapping the multisets of colors by $\mathrm{HASH}$ function: $ c^{(t+1)}(\sigma) = \mathrm{HASH} ( c^{(t)}(\sigma), c^{(t)}_{\mathcal{B}}(\sigma), c^{(t)}_{\mathcal{C}}(\sigma), c^{(t)}_{\mathcal{N}_{\downarrow}}(\sigma), c^{(t)}_{\mathcal{N}_{\uparrow}}(\sigma))$,

3. It stops when a stable coloring is reached. 
\end{Definition}

Fig~\ref{figs001} illustrates the refinement iteratively updates cell colors until the color partition
stabilizes in CCWL. At iteration $t$, CCWL distinguishes $\mathcal{CC}_1$ and $\mathcal{CC}_2$ if $
\{\{c^{(t)}(\sigma)\mid \sigma\in\mathrm{C}_1\}\}
\neq
\{\{c^{(t)}(\tau)\mid \tau\in\mathrm{C}_2\}\}$. 
In this case, the two complexes are declared non-isomorphic. If the color
multisets remain identical until stabilization, the test is inconclusive.

\section{Methods}
In this section, we first define combinatorial complex isomorphism and establish the isomorphism of CCWL. 
We then formalize coloring refinement, show that CCWL subsumes WL variants on graphs, hypergraphs, simplicial complexes, and cell complexes, and prove the sufficiency of lower and upper adjacencies to preserve the distinguishing power of full CCWL under coverage conditions. 
Finally, we instantiate the resulting refinement rule as the Combinatorial Complex Isomorphism Network (CCIN).

\subsection{Combinatorial Complex Isomorphism}

Unlike graph isomorphism, which preserves pairwise adjacency between vertices, 
combinatorial complex isomorphism must preserve both cell ranks and the 
incidence relations among cells of different ranks.

\begin{Definition}[\textbf{Combinatorial Complex Isomorphism}]
\label{dfccin}
Let 
$\mathcal{CC}_1=(\mathrm{S}_1,\mathrm{C}_1,\mathrm{rk}_1)$ and 
$\mathcal{CC}_2=(\mathrm{S}_2,\mathrm{C}_2,\mathrm{rk}_2)$ be two
combinatorial complexes. 
We say that $\mathcal{CC}_1\cong\mathcal{CC}_2$ if there exists a bijection
$\varphi:\mathrm{C}_1\to\mathrm{C}_2$ such that, for every 
$\sigma\in\mathrm{C}_1$,$
\mathrm{rk}_2(\varphi(\sigma))=\mathrm{rk}_1(\sigma),$
$\varphi\big(\mathcal{B}_1(\sigma)\big)
=
\mathcal{B}_{2}\big(\varphi(\sigma)\big)$, $\varphi\big(\mathcal{C}_1(\sigma)\big)
=
\mathcal{C}_{2}\big(\varphi(\sigma)\big)$, 
$\varphi\big(\mathcal{N}_{\uparrow,1}(\sigma)\big)
=
\mathcal{N}_{\uparrow,2}\big(\varphi(\sigma)\big)$,$
\varphi\big(\mathcal{N}_{\downarrow,1}(\sigma)\big)
=
\mathcal{N}_{\downarrow,2}\big(\varphi(\sigma)\big) 
$.
\end{Definition}
% Since lower and upper adjacencies are induced by shared boundary and 
% co-boundary cells, respectively, this definition also preserves 
% $\mathcal{N}_{\downarrow}$ and $\mathcal{N}_{\uparrow}$.
This definition ensures that an isomorphism preserves not only the cell labels
up to relabeling, but also the rank structure and the boundary incidence
relations of the complex. Consequently, the induced boundary, co-boundary,
lower-adjacent, and upper-adjacent neighborhoods are preserved under
$\varphi$. This property is essential for proving the isomorphism invariance
of CCWL color refinement.

\begin{Lemma}
\label{lemma44}
Let $\mathcal{CC}_1$ and $\mathcal{CC}_2$ be two combinatorial complexes. 
If $\mathcal{CC}_1 \cong \mathcal{CC}_2$, then for every iterations $t \geq 0$, the equality of multisets holds $\left \{ \left \{ c^{(t)}(\sigma):\sigma \in \mathrm{C}_1 \right \} \right \} =
\left \{ \left \{ c^{(t)}(\tau):\tau \in \mathrm{C}_2 \right \} \right \}$.
\end{Lemma}

\begin{proof}
By the isomorphism Definition~\ref{dfccin}, there exists a bijection $\varphi : \mathrm{C}_1 \to \mathrm{C}_2$. We prove by induction as Definition~\ref{ccwldef}.
At $t=0$, the initialization rules ensure $c^{(0)}(\sigma) = c^{(0)}(\varphi(\sigma))$. Assuming $c^{(t)}(\sigma) = c^{(t)}(\varphi(\sigma))$ holds at step $t$. For step $t+1$, the CCWL update computes the new color based on the color multisets of $c^{(t)}(\sigma)$. Since $\varphi$ maps the neighborhoods of $\sigma$ to those of $\varphi(\sigma)$, their aggregated input multisets are identical at $t$-th step. Due to injectivity of the $\mathrm{HASH}$ function, then $c^{(t+1)}(\sigma) = c^{(t+1)}(\varphi(\sigma))$. The claim follows by induction. Finally, since $\varphi$ is a bijection on cells,
the color multisets of $\mathcal{CC}_1$ and $\mathcal{CC}_2$ are identical at
every iteration.
\end{proof}

\begin{figure*}[htbp]
 \centering
\includegraphics[width=0.95\linewidth]{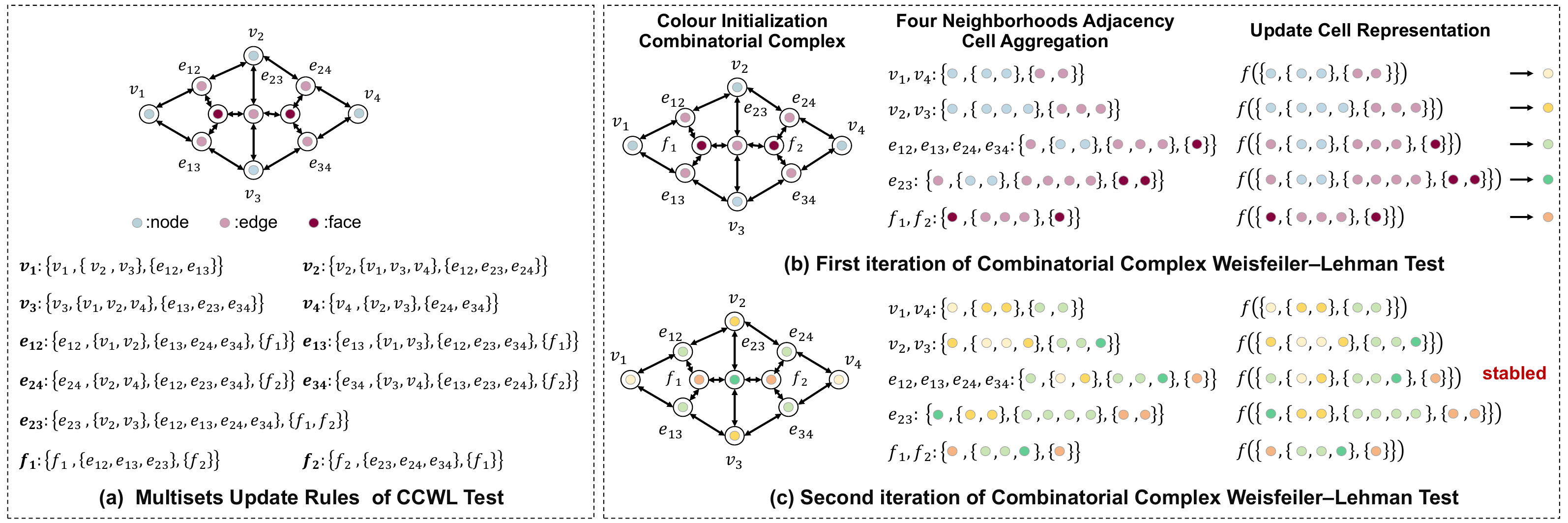}
 \caption{Illustration of Combinatorial Complex Weisfeiler--Leman test. CCWL is performed on a combinatorial complex, which consists of four nodes $v_1,v_2,v_3,v_4$, five edges $e_{12},e_{13},e_{23},e_{24},e_{34}$ and two faces $f_{123}=\{e_{12},e_{13},e_{23}\}$, $f_{234}=\{e_{23},e_{24},e_{34}\}$. At the initialization, 0-cells (nodes) $c_{0}^{(0)}=\{\statebullet\}$, 1-cells (edges) $c_{1}^{(0)}=\{\statepullet\}$, 2-cells (faces) $c_{2}^{(0)}=\{\staterellet\}$. CCWL stabilizes at iteration $t = 2$. }
 \label{fig1}
\end{figure*}

\begin{Definition} 
\label{defb3}
Let $c$ and $d$ be two coloring functions defined on cells of combinatorial complexes. For a coloring \(d\), let $ d^{\mathcal{CC}}:=\{\{d_\sigma^{\mathcal{CC}}\mid \sigma\in\mathcal C\}\}.
$
We define that $c$ refines $d$, denoted by $c\sqsubseteq d$, if for any two
combinatorial complexes $\mathcal{CC}_1,\mathcal{CC}_2$ and any cells $\sigma \in \mathrm{C}_1 $, $\tau \in \mathrm{C}_2 $ , $c_{\sigma}^{\mathcal{CC}_1} = c_{\tau}^{\mathcal{CC}_2}$ implies $d_{\sigma}^{\mathcal{CC}_1} = d_{\tau}^{\mathcal{CC}_2}$ .If both $c\sqsubseteq d$ and $d\sqsubseteq c$ hold, then $c$ and $d$ are equivalent, denoted by $c\equiv d$.
\end{Definition}

\begin{Lemma} \label{lemma9}
 Given two combinatorial complexes $\mathcal{CC}_1$($\mathrm{S}_1,$ $\mathrm{C}_1$,${rk}_1)$, $ \mathcal{CC}_2(\mathrm{S}_2,\mathrm{C}_2,{rk}_2)$, and let 
$A\subseteq\mathrm C_1$ and $B\subseteq\mathcal C_2$.
Let $c$ and $d$ be two colorings such that $c\sqsubseteq d$. If $\{\{d_{\sigma}^{\mathcal{CC}_1} |\sigma\in A\}\}$$ \ne \{\{d_{\tau}^{\mathcal{CC}_2} | \tau\in B\}\}$, then $\{\{c_{\sigma}^{\mathcal{CC}_1} |\sigma \in A\}\} $$\neq \{\{c_{\tau}^{\mathcal{CC}_2} | \tau \in B\}\}$.
\end{Lemma}

\begin{proof}
% We prove the contrapositive. 
Assume $
\left\{\!\left\{c_{\sigma}^{\mathcal{CC}_1}\mid \sigma\in \mathrm C_1\right\}\!\right\}
$=$
\left\{\!\left\{c_{\tau}^{\mathcal{CC}_2}\mid \tau\in\mathrm C_2\right\}\!\right\}
$. 
We suppose two cells $\sigma \in \mathrm{C}_1, \tau \in \mathrm{C}_2$, the multisets of $c$-colorings are 
$
\left\{\!\left\{ c_{\sigma}^{\mathcal{CC}_1} \mid \sigma \in A\right\}\!\right\}$=$\left\{\!\left\{ c_{\tau}^{\mathcal{CC}_2} \mid \tau \in B \right\}\!\right\}$. 
This means there exists a bijection $\phi: A\to B$ such that for all $\sigma \in A$, 
$
c^{\mathcal{CC}_1}_\sigma$=$c^{\mathcal{CC}_2}_{\phi(\sigma)}.
$
Since $c \sqsubseteq d$, we have:
$
c^{\mathcal{CC}_1}_\sigma$=$c^{\mathcal{CC}_2}_{\phi(\sigma)}$, then $d^{\mathcal{CC}_1}_\sigma$=$d^{\mathcal{CC}_2}_{\phi(\sigma)}$.
Thus, for $\forall \sigma \in A$, $d^{\mathcal{CC}_1}_\sigma$=$d^{\mathcal{CC}_2}_{\phi(\sigma)}$, which implies that the multisets of $d$-colors are also equal:
$
\left\{\!\left\{ d_{\sigma}^{\mathcal{CC}_1} \mid \sigma \in A \right\}\!\right\}$=$\left\{\!\left\{ d_{\tau}^{\mathcal{CC}_2} \mid \tau \in B \right\}\!\right\}.
$
This proves the contrapositive, and hence the claim holds.
\end{proof}

\textbf{Remark.}
The relation of two colorings $c\sqsubseteq d$ means that the partition induced by $c$ is finer
than the partition induced by $d$. Equivalently, any distinction made by $d$ is
also made by $c$. Hence, $c$ is at least as discriminative as $d$.

This Lemma further yields a corollary, showing that differences in the coloring function emerge at a refined level.

\begin{Corollary}
\label{Corollary11}
Consider two colorings functions $c, d $ such that $c \sqsubseteq d$. For two combinatorial complexes $\mathcal{CC}_1, \mathcal{CC}_2$ , if $d^{\mathcal{CC}_1} \ne d^{\mathcal{CC}_2}$ , then $c^{\mathcal{CC}_1} \ne c^{\mathcal{CC}_2}$.
\end{Corollary}

\begin{proof}
If $A=\mathrm C_1$ and $B=\mathrm C_2$ in
Lemma~\ref{lemma9}. 
\end{proof}

\subsection{Generalization of CCWL to Four Combinatorial Structures}

Recent studies have extended
WL-type analyses to hypergraphs, simplicial complexes, and cell
complexes~\cite{feng2024hypergraph,bodnar2021weisfeiler,eitantopological}.
However, these refinements are usually defined on different structural domains
with different neighborhood systems. We show that CCWL provides a common
formulation by lifting these structures into combinatorial complexes.

% \begin{Definition}[\textbf{Combinatorial Lifting Map}] \label{def:lifting}
% Let $\mathcal{X}$ be a class of combinatorial structures, $\mathcal{G}, \mathcal{H}, \mathcal{S}, \mathcal{K}$ be a graph, hypergraph, simplicial complex and cellular complex, respectively. A lifting map $f_{\mathcal X}:\mathcal X\to\mathcal{CC}$ is called 
% WL-compatible if it satisfies two properties:
% (i) it preserves isomorphism, i.e.,
% $X_1\cong X_2 \Longleftrightarrow f_{\mathcal X}(X_1)\cong
% f_{\mathcal X}(X_2)$;
% (ii) the WL-type refinement on $X$ can be simulated by CCWL on
% $f_{\mathcal X}(X)$ through the corresponding lifted cells and
% neighborhood relations.
% \end{Definition}

\begin{Definition}[\textbf{Combinatorial Lifting Map}]
\label{def:lifting}
Let $\mathcal{X}$ be a class of combinatorial structures, $\mathcal{G}, \mathcal{H}, \mathcal{S}, \mathcal{K}$ be a graph, hypergraph, simplicial complex and cellular complex, respectively. A combinatorial
lifting map is a transformation
$f_{\mathcal{X}}:\mathcal{X}\rightarrow \mathsf{CC}$ that maps each
$X\in\mathcal{X}$ to a combinatorial complex $f_{\mathcal{X}}(X)$. We call
$f_{\mathcal{X}}$ {faithful} if it satisfies: (i) {isomorphism
preservation and reflection}, i.e., for any $X_1,X_2\in\mathcal{X}$,
$
X_1\cong X_2
\Longleftrightarrow
f_{\mathcal{X}}(X_1)\cong f_{\mathcal{X}}(X_2)$ 
and (ii) {incidence preservation}, i.e., the structural units involved in
the WL-type refinement of $X$ are mapped to cells of corresponding ranks or
types, and their incidence relations are preserved by the boundary or
co-boundary relations in $f_{\mathcal{X}}(X)$.
\end{Definition}

Example of lifting mapping to CC: Let 
$f(\mathcal{\mathcal{X}})$ be the clique-based combinatorial complex, where each node in $\mathcal{X}$ maps to a 0-cell, each edge maps to a 1-cell, each $(k+1)$-clique maps to a k-dimensional cell with rank $k$. 
More broadly, combinatorial complexes offer a common representation for both set-type relations, such as graphs and hypergraphs, and part-whole relations, such as simplicial and cellular complexes. We detail these two cases below.

\begin{figure*}[htbp]
 \centering
\includegraphics[width=0.95\linewidth]{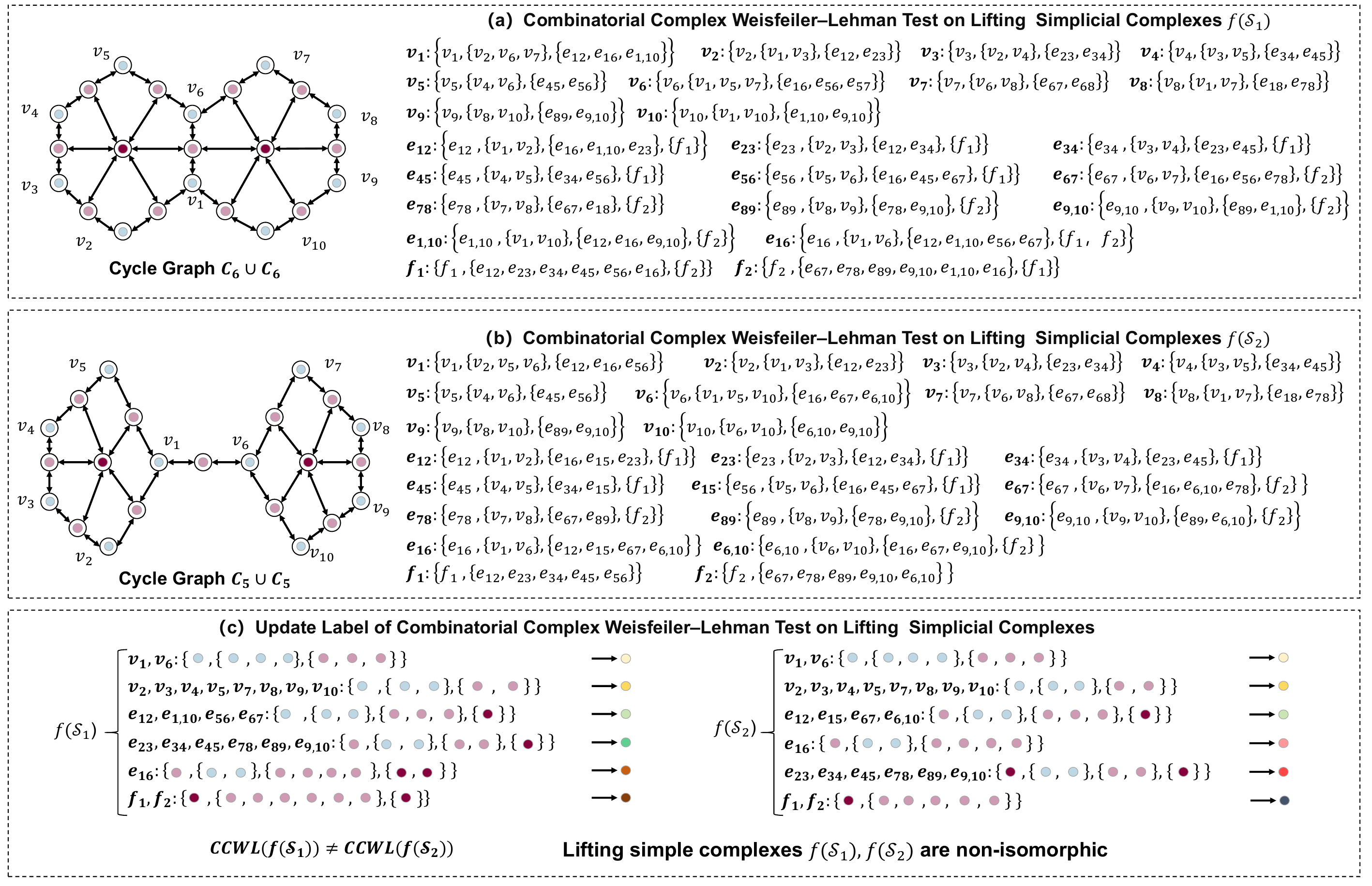}
 \caption{Two simplicial complexes are non-isomorphic ($C_6\cup C_6$: two 6-cycles glued along an edge($e_{16}$), $C_5\cup C_5$: two 5-cycles connected by a bridge edge($e_{16}$)). Simplicial Weisfeiler--Leman cannot distinguish them, but CCWL can distinguish the lifting combinatorial complexes $f(S_1), f(S_2)$.}
 \label{fig004}
\end{figure*}

\refstepcounter{case} 
\textbf{Case~\thecase~(Graphs and Hypergraphs)}
\label{case:graphs-hypergraphs}
Graphs can be viewed as a special case of hypergraphs when the scale of hyperedges degrades to 2. Under the lifting operation, vertices are mapped to $0$-cells, while edges or hyperedges are mapped to higher-rank cells incident to their connected vertices. Hence, the standard graph neighborhood of a vertex can be represented by upper adjacency among $0$-cells through incident edges. For hypergraphs, the two-stage node--hyperedge--node refinement~\cite{zhangimproved,bouritsas2022improving} corresponds to upper adjacency among nodes, it can be rewritten as
% In case of the scale of hyperedges degrading to 2, the hypergraphs can be generalized into graphs.
% Specifically, the node--hyperedge--node update can be represented by upper adjacency:
\begin{equation}
\small
\begin{aligned}
& c^{(t+1)}_v = \mathrm{HASH}\left(c^{(t)}(v), \left\{ c^{(t)}(\tau) \mid \tau \in \mathcal{N}_e(v) \right\}\right) \\ 
\Rightarrow \quad  &c^{(t+1)}_{\mathcal{N}_{\uparrow}(v)} = \mathrm{HASH}\left(c^{(t)}(v), \left\{ (c^{(t)}(\tau), c^{(t)}(\delta)) \right\} \right),
\end{aligned}
\end{equation}
where $\tau$ denotes a hyperedge incident to node $v$ ,$\tau \in \mathcal{N}_{\uparrow}(v), \delta \in \mathcal{C}(v, \tau) $, and $\delta$ refers to a neighboring node. Similarly, the
hyperedge-to-node-to-hyperedge update can be represented by lower adjacency can be expressed as:
\begin{equation}
\small
\begin{aligned}
 & c^{(t+1)}_e = \mathrm{HASH}\left(c^{(t)}(e), \left\{ c^{(t)}(\tau) \mid \tau \in \mathcal{E}(e) \right\}\right)  \\ 
\Rightarrow \quad & c^{(t+1)}_{\mathcal{N}_{\downarrow}(e)} = \mathrm{HASH} \left( c^{(t)}(e), \left\{ (c^{(t)}(\tau), c^{(t)}(\delta)) \right\} \right),
\end{aligned}
\end{equation}
where $\tau \in N_{\downarrow}(e), \delta \in \mathcal{B}(e, \tau)$, $\mathcal{E}(e)$ denotes the set of hyperedges adjacent to $e$ via a shared node, which offers a view of information propagation in graphs and hypergraphs.

\refstepcounter{case} 
\textbf{Case~\thecase~(Simplicial and Cellular complexes)}. 
\label{case:simplicial-cellular}
Simplicial and cell complexes encode part-whole relations through cells of different ranks and their incidence relations~\cite{bodnar2021weisfeiler,eitantopological,wu2023simplicial}. Under the combinatorial lifting, each simplex or cell is mapped to a cell with the same rank, and the boundary incidence relation is preserved. Hence, the boundary, co-boundary, lower-adjacent, and upper-adjacent refinements used in simplicial and cellular WL-type procedures are represented by the four structural neighborhoods of CCWL. For example, consider $2$-simplices (faces)
$f_1$=$\{v_1,v_2,v_3\}$,$f_2$=$\{v_2,v_3,v_4\}, $
which share the common edge $e_{23}=\{v_2,v_3\}$. Their boundary sets are
$
\mathcal{B}(f_1)$=$\{e_{12},e_{23},e_{13}\}$, $
\mathcal{B}(f_2)$=$\{e_{23},e_{24},e_{34}\}.
$
The shared edge $e_{23}$ induces the lower adjacency between the two faces, since 
$\mathcal{B}(f_1)\cap \mathcal{B}(f_2)=\{e_{23}\}\neq\emptyset,$ $f_1\in\mathcal{N}_{\downarrow}(f_2).$ Similarly, the co-boundary of the shared edge satisfies $\mathcal{C}(e_{23})=\{f_1,f_2\},$ showing that $e_{23}$ connects the two incident faces through a cross-rank part-whole relation. At lower ranks, two vertices are upper-adjacent if they belong to a common edge, e.g., $v_2\in\mathcal{N}_{\uparrow}(v_3)$ through the co-boundary cell $e_{23},$ and two edges are upper-adjacent if they belong to a common face, e.g., $e_{12}\in\mathcal{N}_{\uparrow}(e_{13})$ through the co-boundary cell $f_1$. Thus, the neighborhood systems used by simplicial and cellular refinements can be viewed as specific instances of the four-neighborhood CCWL refinement.

\begin{Theorem}
\label{theorem06}
Let $\mathcal{G}, \mathcal{H}, \mathcal{S}, \mathcal{K}$ be a graph, hypergraph, simplicial complex and cellular complex, respectively. There exists a mapping from each of these structures into a combinatorial complex $\mathcal{CC}$, the CCWL test can simulate the 1-WL, hypergraph WL, simplicial-WL and cellular-WL tests. 
\end{Theorem}
 
\begin{proof}
For each structure, we use the combinatorial lifting in
Definition~\ref{def:lifting}. Vertices are lifted to $0$-cells, edges or
hyperedges, simplices are lifted to incident higher-rank cells. In Case~\ref{case:graphs-hypergraphs}, graph and hypergraph WL refinements are represented in CCWL through lifted incidence relations: graph vertex neighborhoods correspond to upper-adjacent $0$-cells, hypergraph node--hyperedge--node and hyperedge--node--hyperedge refinements correspond to upper adjacency among nodes and lower adjacency among hyperedges, respectively. In Case~\ref{case:simplicial-cellular}, simplicial- and cellular-WL refinements are embedded into CCWL, since the lifting preserves cells, ranks, and boundary maps, making their boundary, co-boundary, lower-adjacent, and upper-adjacent color multisets coincide with the corresponding CCWL neighborhood multisets. At every iteration, the information used by each domain-specific
WL update is contained in the information used by the CCWL update. 
By the
injectivity of $\mathrm{HASH}$ and induction over iterations, the CCWL coloring
refines the corresponding domain-specific WL coloring, i.e.,$
c_{\mathrm{CCWL}}^{(t)}\sqsubseteq c_{\mathrm{WL}}^{(t)}$. 
Therefore, by Lemma~\ref{lemma9} on the relevant lifted cell subsets, and by
Corollary~\ref{Corollary11} when the comparison is over the whole cell set, any
distinction made by 1-WL, hypergraph WL, simplicial-WL, or cellular-WL is also
made by CCWL on the lifted combinatorial complex. Hence, CCWL simulates these
WL-type tests.
\end{proof}

We further provide representative non-isomorphic pairs to illustrate the
discriminative gap between existing WL-type refinements and CCWL. Under the
simplicial-WL refinement, message aggregation is restricted to the original simplicial incidence relations. Starting from rank-wise uniform initial colors, pairs such as $C_k\cup C_k$ and $C_{2k}$ $(k\geq 4)$, $C_{k-1}\cup C_{k-1}$ with a bridge edge and $C_{k-1}\cup C_{k-1}$ with a shared edge ($k\geq 5$), induce identical stable color histograms over vertices and edges. Hence, for such a pair
$(\mathcal S_1,\mathcal S_2)$,
$
\left\{\!\left\{
c_{\mathrm{SWL}}^{(t)}(\sigma)\mid \sigma\in\mathcal S_1
\right\}\!\right\}
$=$
\left\{\!\left\{
c_{\mathrm{SWL}}^{(t)}(\tau)\mid \tau\in\mathcal S_2
\right\}\!\right\},$
for all iterations before stabilization, and simplicial-WL fails to distinguish
the two structures. 
After applying the cyclic lifting $f(\cdot)$, the structures are represented as
combinatorial complexes whose lifted cells encode closed cycle patterns. CCWL can then distinguish the two lifted complexes by aggregating bridge-aware lower- and upper-adjacent information. At some finite iteration
$t$, $
\{\{
c_{\mathrm{CCWL}}^{(t)}(\sigma)\mid
\sigma\in\mathcal C(f(\mathcal S_1))
\}\}
$$\neq$$
\{\{
c_{\mathrm{CCWL}}^{(t)}(\tau)\mid
\tau\in\mathcal C(f(\mathcal S_2))
\}\}$. 
Fig~\ref{fig004} illustrates one such separation example, and additional
counterexamples are provided in SI~Sec.II.

The full CCWL refinement uses four structural neighborhoods: boundary, co-boundary, lower adjacency, and upper adjacency. These neighborhoods capture different aspects of a combinatorial complex, but they are not independent. For two same-rank cells $\sigma$ and $\tau$, recall the bridge sets $ \mathcal{B}(\sigma,\tau) := \mathcal{B}(\sigma)\cap\mathcal{B}(\tau) $, $\mathcal{C}(\sigma,\tau) := \mathcal{C}(\sigma)\cap\mathcal{C}(\tau)$. 
Lower adjacency $\mathcal{B}(\sigma,\tau)$ is induced by shared boundary cells $\mathcal{B}(\sigma)\cap\mathcal{B}(\tau)$, whereas upper adjacency $\mathcal{C}(\sigma,\tau)$ is induced by shared co-boundary cells $\mathcal{C}(\sigma)\cap\mathcal{C}(\tau)$ (i.e., $c_{\mathcal{N}_\uparrow}^{(t)} \sqsubseteq c_{\mathcal{C}}^{(t)}$), and lower adjacent colorings encompass boundary colorings (i.e., $c_{\mathcal{N}_\downarrow}^{(t)} \sqsubseteq c_{\mathcal{B}}^{(t)}$).
% Hence, part of the cross-rank incidence information may be recoverable from same-rank adjacency relations. 
we formally derive the neighborhood coverage conditions as

\begin{Definition}[\textbf{Coverage Condition}]
\label{def:coverage}
Let $\sigma\in\mathcal{C}$ be a $k$-cell in a combinatorial complex. 
We say that $\sigma$ is boundary-covered if every boundary cell of
$\sigma$ is shared with another $k$-cell; that is, for every
$\eta\in\mathcal{B}(\sigma)$, there exists $\tau\in\mathcal{C}$ such that
$\tau\neq\sigma$, $\mathrm{rk}(\tau)=\mathrm{rk}(\sigma)$, and
$\eta\in\mathcal{B}(\sigma)\cap\mathcal{B}(\tau)$. 
Similarly, $\sigma$ is co-boundary-covered if every co-boundary cell of
$\sigma$ is shared with another $k$-cell; that is, for every
$\gamma\in\mathcal{C}(\sigma)$, there exists $\tau\in\mathcal{C}$ such that $\tau\neq\sigma$, $\mathrm{rk}(\tau)=\mathrm{rk}(\sigma)$, and
$\gamma\in\mathcal{C}(\sigma)\cap\mathcal{C}(\tau)$.
\end{Definition} 

The coverage condition rules out boundary or co-boundary cells that are visible  through cross-rank incidence. Under this condition, the boundary and co-boundary neighborhoods can be reconstructed from their corresponding bridge sets.

\begin{Lemma}
\label{lemma:equivalence_coverage}
Let $\sigma$ be a $k$-cell in a combinatorial complex $\mathcal{CC}$. 
If $\sigma$ is both boundary-covered and co-boundary-covered, then 
\begin{align} \small
\label{covercond01}
\mathcal{B}(\sigma) & = \bigcup_{\tau \in \mathcal{N}_{\downarrow}(\sigma)} \mathcal{B}(\sigma, \tau)=\bigcup_{\tau \in \mathcal{N}_{\downarrow}(\sigma)} \mathcal{B}(\sigma)\cap \mathcal{B}(\tau),  \\ 
\label{covercond02}
\mathcal{C}(\sigma) & = \bigcup_{\tau \in \mathcal{N}_{\uparrow}(\sigma)} \mathcal{C}(\sigma, \tau) = \bigcup_{\tau \in \mathcal{N}_{\uparrow}(\sigma)} \mathcal{C}(\sigma) \cap \mathcal{C}(\tau).
\end{align} 

\end{Lemma}

\begin{proof}
We prove the lemma with mutual inclusion (i.e., $\supseteq$, $\subseteq$). The bridge sets are the shared cells between two adjacent $k$-cells: $\mathcal{B}(\sigma, \tau) = \mathcal{B}(\sigma) \cap \mathcal{B}(\tau)$ and $\mathcal{C}(\sigma, \tau) = \mathcal{C}(\sigma) \cap \mathcal{C}(\tau)$. Assume cell $\sigma$ is a boundary-covered $k$-cell. 
($\supseteq$): Let $\eta \in \bigcup_{\tau \in \mathcal{N}_{\downarrow}(\sigma)} (\mathcal{B}(\sigma) \cap \mathcal{B}(\tau))$. By definition, there exists a lower adjacent cell $\tau \in \mathcal{N}_{\downarrow}(\sigma)$ such that $\eta \in \mathcal{B}(\sigma) \cap \mathcal{B}(\tau)$. It trivially follows that $\eta \in \mathcal{B}(\sigma)$. Thus, $\bigcup_{\tau \in \mathcal{N}_{\downarrow}(\sigma)} (\mathcal{B}(\sigma) \cap \mathcal{B}(\tau)) \subseteq \mathcal{B}(\sigma)$. 
($\subseteq$): Let $\eta \in \mathcal{B}(\sigma)$ be a boundary cell of $\sigma$. Because $\sigma$ satisfies the boundary-covered condition (Eq.~\eqref{covercond01}), there exists another $k$-cell $\tau \neq \sigma$ such that $\eta \in \mathcal{B}(\sigma) \cap \mathcal{B}(\tau)$. This non-empty intersection strictly fulfills the definition of lower adjacency, yielding $\tau \in \mathcal{N}_{\downarrow}(\sigma)$. Hence, $\eta \in \bigcup_{\tau \in \mathcal{N}_{\downarrow}(\sigma)} (\mathcal{B}(\sigma) \cap \mathcal{B}(\tau))$, establishing $\mathcal{B}(\sigma) \subseteq \bigcup_{\tau \in \mathcal{N}_{\downarrow}(\sigma)} (\mathcal{B}(\sigma) \cap \mathcal{B}(\tau))$. Both inclusions are satisfied, $\mathcal{B}(\sigma) = \bigcup_{\tau \in \mathcal{N}_{\downarrow}(\sigma)} \mathcal{B}(\sigma, \tau)$ holds.
For the co-boundary case, if $\gamma\in\mathcal{C}(\sigma)$ and $\sigma$ is co-boundary-covered, then there exists a $k$-cell $\tau\neq\sigma$ such that $ \gamma\in\mathcal{C}(\sigma)\cap\mathcal{C}(\tau)$. Hence $\tau\in\mathcal{N}_{\uparrow}(\sigma)$ and $\gamma\in\mathcal{C}(\sigma,\tau)$. Together with the trivial reverse inclusion, this proves Eq.~\eqref{covercond02}. This completes the proof. 
\end{proof}
 
% Assume $\sigma$ is a $k$-cell, co-boundary-covered.
% ($\supseteq$): Let $\gamma \in \bigcup_{\tau \in \mathcal{N}_{\uparrow}(\sigma)} \mathcal{C}(\sigma, \tau)$. This implies there exists an upper adjacent cell $\tau \in \mathcal{N}_{\uparrow}(\sigma)$ such that $\gamma \in \mathcal{C}(\sigma, \tau)$. Since $\mathcal{C}(\sigma, \tau) = \mathcal{C}(\sigma) \cap \mathcal{C}(\tau)$, it is evident that $\gamma \in \mathcal{C}(\sigma)$. Thus, $\bigcup_{\tau \in \mathcal{N}_{\uparrow}(\sigma)} \mathcal{C}(\sigma, \tau) \subseteq \mathcal{C}(\sigma)$. 
% ($\subseteq$): Let $\gamma \in \mathcal{C}(\sigma)$ be a co-boundary cell of $\sigma$. According to the co-boundary-covered condition (Definition \ref{def:coverage}), there exists another $k$-cell $\tau \neq \sigma$ such that $\gamma \in \mathcal{C}(\sigma) \cap \mathcal{C}(\tau)$. This non-empty intersection defines the upper adjacency between $\sigma$ and $\tau$, ensuring $\tau \in \mathcal{N}_{\uparrow}(\sigma)$. Because $\gamma \in \mathcal{C}(\sigma, \tau)$, it follows that $\gamma \in \bigcup_{\tau \in \mathcal{N}_{\uparrow}(\sigma)} \mathcal{C}(\sigma, \tau)$. Thus, $\mathcal{C}(\sigma) \subseteq \bigcup_{\tau \in \mathcal{N}_{\uparrow}(\sigma)} \mathcal{C}(\sigma, \tau)$.
% Both inclusions are satisfied, yielding the equality $\mathcal{C}(\sigma) = \bigcup_{\tau \in \mathcal{N}_{\uparrow}(\sigma)} \mathcal{C}(\sigma, \tau)$. 

\subsection{Combinatorial Complex Weisfeiler--Leman Framework} 
To further explore whether all four neighborhood color multisets are necessary to maintain the discriminative power of CCWL, and whether a simplified refinement rule can achieve the same expressive power, we analyze this redundancy by removing the boundary and co-boundary color multisets.
% To further explore whether all four neighborhood color multisets are necessary to maintain the discriminative power of CCWL, and whether a simplified refinement rule can achieve the same expressive power, we first analyze this redundancy by removing the shared-boundary color multisets while retaining boundary, lower adjacency, and upper adjacency information. Furthermore, removing the boundary color multisets provides the necessary neighborhood information.

\begin{Lemma}
\label{lemma12}
For combinatorial complexes satisfying the coverage conditions, 
 CCWL with $\mathrm{HASH} (c_\sigma^t$, $c_{\mathcal{B}}^t(\sigma)$, $c_{\mathcal{N}_{\downarrow}}^t(\sigma)$, $ c_{\mathcal{N}_{\uparrow}}^t(\sigma) )$ is as powerful as CCWL with the generalized update rule $\mathrm{HASH} (c^t(\sigma), c_{\mathcal{B}}^t(\sigma)$, $c_{\mathcal{C}}^t(\sigma)$, $c_{\mathcal{N}_{\downarrow}}^t(\sigma)$, $c_{\mathcal{N_{\uparrow}}}^t(\sigma) )$ .
\end{Lemma}

\begin{proof} 
Let $\mathcal{CC}_1 $, $\mathcal{CC}_2$ be two combinatorial complexes. The colorings rules $\mathrm{HASH} (c_\sigma^t, c_{\mathcal{B}}^t(\sigma), c_{\mathcal{C}}^t(\sigma), c_{\mathcal{N}_{\downarrow}}^t(\sigma), c_{\mathcal{N}_{\uparrow}}^t(\sigma) )$, $\mathrm{HASH} (c_\sigma^t, c_{\mathcal{B}}^t(\sigma), c_{\mathcal{N}_{\downarrow}}^t(\sigma), c_{\mathcal{N}_{\uparrow}}^t(\sigma) )$ denote as $a^t$ and $b^t$ at iteration $t$, respectively. 
Since $a^t$ incorporates the additional co-boundary colors $c_{\mathcal{C}}^t(\sigma)$, it means that $a^t \sqsubseteq b^t$ (Definition \ref{defb3}). By induction, if we prove $b^t \sqsubseteq a^t$, it then implies $a^t \equiv b^t$. 
For the base case $t=0$, all cells are initialized identically, $b^{(0)} \sqsubseteq a^{(0)}$ holds trivially.  Assume $\sigma \in \mathrm{C}_1$ and $\tau \in \mathrm{C}_2$ are cells of the same rank such that $b_\sigma^{t+1} = b_\tau^{t+1}$. Thus, we have $b_\sigma^t = b_\tau^t$, $b_{\mathcal{N}_\downarrow}^t(\sigma) = b_{\mathcal{N}_\downarrow}^t(\tau)$, $b_{\mathcal{N}_\uparrow}^t(\sigma) = b_{\mathcal{N}_\uparrow}^t(\tau)$, and $b_{\mathcal{B}}^t(\sigma) = b_{\mathcal{B}}^t(\tau)$. 
To complete the induction, we next prove that these equalities imply $b_{\mathcal{C}}^t(\sigma) = b_{\mathcal{C}}^t(\tau)$. 

If coverage conditions are satisfied, Lemma \ref{lemma:equivalence_coverage} Eq.~\eqref{covercond02} shows the co-boundary can be recovered from the upper adjacency bridges: $\mathcal{C}(\sigma) = \bigcup_{\tau \in \mathcal{N}_{\uparrow}(\sigma)} \mathcal{C}(\sigma, \tau)$.   
By the definition of the upper adjacency $\mathcal{N}_{\uparrow}(\sigma)$ ($c_{\mathcal{N}_{\uparrow}}(\sigma) = \{\{(c_{\tau}, c_{\delta}) \mid \tau \in \mathcal{N}_{\uparrow}(\sigma)$ and $\delta \in \mathcal{C}(\sigma, \tau)\}\}$), 
it holds $b_{\mathcal{N}_\uparrow}^t(\sigma) = b_{\mathcal{N}_\uparrow}^t(\tau)$ that ensures that
$
  \{\{ b_{\delta_\sigma}^{t} \mid (\cdot, b_{\delta_\sigma}^{t}) \in b_{\mathcal{N}_\uparrow}^{t}(\sigma) \}\}$
=
$\{\{ b_{\delta_\tau}^{t} \mid (\cdot, b_{\delta_\tau}^{t}) \in b_{\mathcal{N}_\uparrow}^{t}(\tau) \}\}$. 
Each coface $\delta_\sigma \in \mathcal{C}(\sigma)$ contains exactly $|\mathcal{B}(\delta_\sigma)|$ boundary cells. Since $\sigma$ is one of them, $\delta_\sigma$ contributes $|\mathcal{B}(\delta_\sigma)| - 1$ tuples to the upward multiset of $\sigma$. 
Because the boundary size $|\mathcal{B}(\delta_\sigma)|$ is encoded in the coloring $b_{\delta_\sigma}^{t}$, we partition the multisets by these degrees. By dividing out the repeated entries corresponding to these multiplicities, $
\{\{ b_{\delta_\sigma}^{t} \mid \delta_\sigma \in \mathcal{C}(\sigma) \}\}
$=$\{\{ b_{\delta_\tau}^{t} \mid \delta_\tau \in \mathcal{C}(\tau) \}\}$ implies $b_{\mathcal{C}}^t(\sigma) = b_{\mathcal{C}}^t(\tau)$. By the induction hypothesis, we further have $a_\sigma^t = a_\tau^t$, $a_{\mathcal{N}_\downarrow}^t(\sigma) = a_{\mathcal{N}_\downarrow}^t(\tau)$, $a_{\mathcal{N}_\uparrow}^t(\sigma) = a_{\mathcal{N}_\uparrow}^t(\tau)$, $a_{\mathcal{B}}^t(\sigma) = a_{\mathcal{B}}^t(\tau)$, and $a_{\mathcal{C}}^t(\sigma) = a_{\mathcal{C}}^t(\tau)$. Therefore, $a_\sigma^{t+1} = a_\tau^{t+1}$.
\end{proof}

In the following theorem, we assume WL-compatible initialization, where the initial color of each cell encodes its rank and the sizes of its boundary and co-boundary neighborhoods.

\begin{Theorem}
\label{theorem02} 
CCWL with 
$\mathrm{HASH}$ $(c^{(t)}_\sigma,  c^{(t)}_{\mathcal{N}_{\downarrow}}(\sigma), c^{(t)}_{\mathcal{N}_{\uparrow}}(\sigma) )$ is as powerful as CCWL with the generalized update rule $\mathrm{HASH}\big(c^{(t)}(\sigma),  c^{(t)}_{\mathcal{B}}(\sigma)$, $c^{(t)}_{\mathcal{C}}(\sigma)$, $c^{(t)}_{\mathcal{N}_{\downarrow}}(\sigma)$,  $c^{(t)}_{\mathcal{N}_{\uparrow}}(\sigma)\big)$ if combinatorial complexes satisfy the coverage conditions. 
\end{Theorem}

\begin{proof}
Let $a^t$ and $b^t$ denote the colorings produced by the full and restricted CCWL updates $\mathrm{HASH}\big(c^{(t)}(\sigma)$, $c^{(t)}_{\mathcal{B}}(\sigma)$, $c^{(t)}_{\mathcal{C}}(\sigma),$$ c^{(t)}_{\mathcal{N}_{\downarrow}}(\sigma), c^{(t)}_{\mathcal{N}_{\uparrow}}(\sigma)\big)$, and $\mathrm{HASH} \big(c^{(t)}_\sigma, c^{(t)}_{\mathcal{N}_{\downarrow}}(\sigma), c^{(t)}_{\mathcal{N}_{\uparrow}}(\sigma) \big)$, respectively. Since the full update $a^t$ contains all information used by the restricted update $b^t$, we have $a^t\sqsubseteq b^t$. It remains to prove the reverse refinement $b^t\sqsubseteq a^t$ by induction, which yields $a^t\equiv b^t$. For the base case $t=0$, the initial colors are determined by $c^{(0)}(\sigma) = \mathrm{HASH}_0(\mathrm{rk}(\sigma), |\mathcal{B}(\sigma)|, |\mathcal{C}(\sigma)|)$. Since the boundary and co-boundary sizes are implicitly deterministic for both schemes at initialization, we have $a^0 \equiv b^0$.

Suppose that $b^t \sqsubseteq a^t$ holds for some iteration $t$. Assume $\sigma \in \mathrm{C}_1$ and $\tau \in \mathrm{C}_2$ are cells of the same rank such that $b_\sigma^{t+1}=b_\tau^{t+1}$. 
% By the injectivity of $\mathrm{HASH}$ in the restricted update, its input arguments must coincide. 
By the restricted update, $b_\sigma^t=b_\tau^t$, $b_{\mathcal N_\downarrow}^t(\sigma) = b_{\mathcal N_\downarrow}^t(\tau)$, and $b_{\mathcal N_\uparrow}^t(\sigma) = b_{\mathcal N_\uparrow}^t(\tau)$. 
% To prove that the restricted rule $b^t$ also determines the boundary and co-boundary color multisets, namely $b_{\mathcal B}^t(\sigma) = b_{\mathcal B}^t(\tau)$ and $b_{\mathcal C}^t(\sigma) = b_{\mathcal C}^t(\tau)$. 
Under the coverage conditions, Lemma~\ref{lemma12} recovers the co-boundary  multisets from the upper-adjacent, which yields $b_{\mathcal C}^t(\sigma) = b_{\mathcal C}^t(\tau)$. 
We next recover the boundary color multiset. By the definition of lower-neighborhood color multisets, $b_{\mathcal N_\downarrow}^t(\sigma) = \left\{\!\left\{ \big(b_\tau^t, b_\eta^t\big) \;\middle|\; \tau\in\mathcal N_\downarrow(\sigma), \eta\in\mathcal B(\sigma,\tau) \right\}\!\right\},$ where $\mathcal B(\sigma,\tau) = \mathcal B(\sigma)\cap \mathcal B(\tau)$ is the shared boundary bridge set of $\sigma$ and $\tau$. Since $\sigma$ is boundary-covered, Lemma \ref{lemma:equivalence_coverage} Eq.~\eqref{covercond01} implies $\mathcal B(\sigma) = \bigcup_{\tau\in\mathcal N_\downarrow(\sigma)} \mathcal B(\sigma,\tau)$. Each boundary cell $\eta\in\mathcal B(\sigma)$ appears $|\mathcal C(\eta)|-1$ times, because $\eta$ is incident to $\sigma$ and to $|\mathcal C(\eta)|-1$ other cells of the same rank as $\sigma$. Each such cell is lower adjacent to $\sigma$ through $\eta$. 
% Again, the repetition factor $|\mathcal C(\eta)|-1$ is determined by the color $b_\eta^t$, since the initial coloring encodes $|\mathcal C(\eta)|$ and the subsequent updates are injective refinements. 
Therefore, normalizing the multiplicity of each color by its encoded repetition factor recovers $ b_{\mathcal B}^t(\sigma) = \left\{\!\left\{ b_\eta^t \;\middle|\; \eta\in\mathcal B(\sigma) \right\}\!\right\}.$ Since $b_{\mathcal N_\downarrow}^t(\sigma) = b_{\mathcal N_\downarrow}^t(\tau)$, it gives $b_{\mathcal B}^t(\sigma) = b_{\mathcal B}^t(\tau)$.

We have now shown that $b_\sigma^t=b_\tau^t$, $b_{\mathcal B}^t(\sigma) = b_{\mathcal B}^t(\tau)$, $b_{\mathcal C}^t(\sigma) = b_{\mathcal C}^t(\tau)$, $b_{\mathcal N_\downarrow}^t(\sigma) = b_{\mathcal N_\downarrow}^t(\tau)$, and $b_{\mathcal N_\uparrow}^t(\sigma) = b_{\mathcal N_\uparrow}^t(\tau)$. By the induction hypothesis $b^t\sqsubseteq a^t$, equality of $b^t$-colored cells implies equality of the corresponding $a^t$-colored cells. Hence the corresponding full-rule arguments also agree: $a_\sigma^t=a_\tau^t$, $a_{\mathcal B}^t(\sigma) = a_{\mathcal B}^t(\tau)$, $a_{\mathcal C}^t(\sigma) = a_{\mathcal C}^t(\tau)$, $a_{\mathcal N_\downarrow}^t(\sigma) = a_{\mathcal N_\downarrow}^t(\tau)$, and $a_{\mathcal N_\uparrow}^t(\sigma) = a_{\mathcal N_\uparrow}^t(\tau)$.Therefore, the full generalized update receives identical inputs for $\sigma$ and $\tau$ at iteration $t$, and the injectivity of $\mathrm{HASH}$ yields $a_\sigma^{t+1} = a_\tau^{t+1}$. Thus, $b^{t+1}\sqsubseteq a^{t+1}$. By induction, $b^t\sqsubseteq a^t$ for all $t$. Together with $a^t\sqsubseteq b^t$, we obtain $a^t\equiv b^t$ for all iterations $t$. 
% Therefore, under the coverage conditions, the restricted CCWL update $\mathrm{HASH}\big(c_\sigma^t, c_{\mathcal N_\downarrow}^t(\sigma), c_{\mathcal N_\uparrow}^t(\sigma)\big)$ is as powerful as the full generalized update $\mathrm{HASH}\big(c_\sigma^t, c_{\mathcal B}^t(\sigma), c_{\mathcal C}^t(\sigma), c_{\mathcal N_\downarrow}^t(\sigma), c_{\mathcal N_\uparrow}^t(\sigma)\big)$.
\end{proof}

% Building on this, we can derive this lemma as
% \begin{Lemma}
% \label{stabilitycond}
% CCWL is at least as powerful as 1-WL in distinguishing of non-isomorphic combinatorial complexes.
% \end{Lemma}

% \begin{proof}[\textbf{Proof for Lemma} \ref{stabilitycond}]

% Let $\mathcal{CC}_1$ and $\mathcal{CC}_2$ be two combinatorial complexes. Let $c$ denote the stable coloring produced by a WL-style test (e.g., a restricted version operating only on 0-cells or using fewer aggregation rules), and let $d$ denote the stable coloring produced by CCWL. 
% Since CCWL aggregates over more structural information. Specifically, it considers both lower neighbourhoods $\mathcal{N}_\downarrow$, upper neighbourhoods $\mathcal{N}_\uparrow$, boundary and coface relations, its coloring $d$ distinguishes more cells than $c$. Therefore, the partition induced by $c$ is coarser than that induced by $d$, i.e., $ c \sqsubseteq d.$ 
% Suppose that the WL test distinguishes $\mathcal{CC}_1$ and $\mathcal{CC}_2$. That is, $
% c^{\mathcal{CC}_1} \ne c^{\mathcal{CC}_2},
% $ which means the multiset of colors differs between the two complexes. 
% By Corollary~\ref{Corollary11}, since $c \sqsubseteq d$, we have:
%  $
% \text{If } c^{\mathcal{CC}_1} \ne c^{\mathcal{CC}_2}, \text{ then } d^{\mathcal{CC}_1} \ne d^{\mathcal{CC}_2} $. Therefore, CCWL also distinguishes $\mathcal{CC}_1$ and $\mathcal{CC}_2$.
% This holds for any pair distinguishable by the WL test. Hence, CCWL is at least as powerful as the WL test in distinguishing non-isomorphic combinatorial complexes.

% \end{proof}

\subsection{Combinatorial Complex Isomorphism Network}
We instantiate the preceding CCWL analysis as a neural message-passing model, termed the Combinatorial Complex Isomorphism Network (CCIN). Rather than introducing an additional neighborhood system, CCIN follows the reduced CCWL refinement justified by Theorem~\ref{theorem02}. This design uses only lower- and upper-adjacent bridge information while preserving the discriminative power of the full four-neighborhood CCWL under the coverage conditions.
Let $\mathcal{CC}=(\mathcal S,\mathcal C,\mathrm{rk})$ be a combinatorial complex, and let $N_{\mathcal{C}}= \{\mathcal B,\mathcal C,\mathcal N_{\downarrow},\mathcal N_{\uparrow}\}$ denote the four structural neighborhood functions. 
% The $l$-th layer of CCNN updates the representation $h_\sigma^{(l)}\in\mathbb R^{F_l}$ of a cell $\sigma\in\mathcal C$ by
% \begin{equation}
% \label{eq:ccnn_general} 
% h_{\sigma}^{(l+1)} = \phi^{(l)} \left( h_{\sigma}^{(l)}, \bigotimes_{\mathcal N\in\mathcal N_{\mathcal C}} \mathrm{AGG}_{\tau\in\mathcal N(\sigma)} M_{\mathcal N}^{(l)} \big(h_\sigma^{(l)},h_\tau^{(l)}\big) \right),
% \end{equation}
% where $M_{\mathcal N}^{(l)}$ is a neighborhood-specific message function, $\mathrm{AGG}$ is a permutation-invariant aggregation operator, and $\bigotimes$ combines messages from different neighborhood types.
From the four CCWL neighborhoods in Definition \ref{cellneighhood}, their messages aggregation with the input features $h_\sigma$ can be written as
\begin{equation}
\small
\begin{aligned}
m_{\mathcal{B}}^{t+1}(\sigma) & = \mathop{\operatorname{AGG}}_{\tau \in \mathcal{B}(\sigma)}\left(M_{\mathcal{B}}\left(h_\sigma^t, h_\tau^t\right)\right),\\
m_{\mathcal{C}}^{t+1}(\sigma) & = \mathop{\operatorname{AGG}}_{\tau \in \mathcal{C}(\sigma)}\left(M_{\mathcal{C}}\left(h_\sigma^t, h_\tau^t\right)\right),\\
m^{(t+1)}_{\mathcal{N}_{\downarrow}} (\sigma) & = \mathop{\text{AGG}}_{\tau\in \mathcal{N}_{\downarrow},\delta\in\mathcal{B}(\sigma, \tau) } \left ( M_{\mathcal{N}_{\downarrow}} (h_\sigma^{(t)}, h_\tau^{(t)},h_\delta^{(t)} ) \right ), \\
 m^{(t+1)}_{\mathcal{N}_{\uparrow}}(\sigma) & = \mathop{\text{AGG}}\limits_{\tau\in \mathcal{N}_{\uparrow},\delta\in\mathcal{C}(\sigma, \tau) } \left ( M_{\mathcal{N}_{\uparrow}} (h_\sigma^{(t)}, h_\tau^{(t)},h_\delta^{(t)} ) \right ). 
\end{aligned}
\end{equation}

Guided by Theorem~\ref{theorem02}, under the coverage conditions, the boundary and co-boundary color multisets can be recovered from lower- and upper-adjacent bridge information.
SI Sec.III supplements the theoretical connection between CCIN and reduced CCWL. The cells in CCNNs receive two types of messages, then update operation takes into account these two types of messages and updates the features of the cells:
\begin{equation}
\label{equation05}
 \begin{aligned}
 h_{\sigma}^{(t+1)} = \text{Update} \left ( h_{\sigma}^{(t)}, m^{(t+1)}_{\mathcal{N}_{\downarrow}}(\sigma), m^{(t+1)}_{\mathcal{N}_{\uparrow}}(\sigma)\right ).
\end{aligned}
\end{equation}
To obtain a global embedding from a $d$-dimension combinatorial complex $\mathcal{CC}$ of a CCNN with $L$ layers, the readout function takes as input the sets of features corresponding to all dimensions of the combinatorial complex
\begin{equation}
\small
\begin{aligned}
 h_{\mathcal{CC}} = \text{READOUT} ( \{ \{ h_\sigma^{(L)} \} \}_{\text{dim}=0},\dots, \{ \{ h_\sigma^{(L)} \} \}_{\text{dim}=d} ),
\end{aligned} 
\end{equation}
where $d=\dim(\mathcal{CC})$. The \text{READOUT} can be instantiated by any permutation-invariant pooling function, such as SUM, MEAN, or MAX pooling function.

\section{Experimental Analysis}

To evaluate the performance of the proposed combinatorial complex isomorphism network, extensive experiments are conducted on synthetic and real-world datasets. In the following, we introduce experimental results, ablation study.

\subsection{Experimental Settings}

\begin{table}[htbp]
% \scriptsize 
\fontsize{6}{7.2}\selectfont 
\centering
\caption{Details of Strongly Regular Families}
\setlength{\tabcolsep}{4pt} 
\begin{tabular}{lccccccccc}
\hline Family &(16,6,2,2)&(25,12,5,6)&(26,10,3,4)& (28,12,6,4)&(29,14,6,7) \\
\hline 
Graphs Number & 2 & 15 & 10 & 4 & 41 \\
\hline Family &  (35,16,6,8)&(35,18,9,9)&(36,14,4,6)&(40,12,2,4) \\
\hline 
Graphs Number &  3854 & 227 & 180 & 28 \\

\hline
\end{tabular}
\label{tab:srdatasets}
\end{table}

\begin{table}[htbp]
% \scriptsize 
\fontsize{6}{7.2}\selectfont 
\centering
\setlength{\tabcolsep}{1pt} 
\centering
\caption{Overview of the graph learning datasets.}
\label{tab:realdatasets}
\begin{threeparttable}
\begin{tabular}{cccccccccc}
 \toprule 
\multicolumn{2}{c}{Dataset} & \# Graphs & Avg. nodes & Avg. edges & Prediction task & Class  & Metric \\
\midrule
\multirow{9}*{\rotatebox{90}{TUDataset}}  & MUTAG & 188 & 17.93 & 19.79 & Classification & 2 & Accuracy \\
 & PTC & 344 & 25.56 & 25.96 & Classification & 2 & Accuracy \\
 & NCI1 & 4110 & 29.87 & 32.30 & Classification & 2 & Accuracy \\
 & NCI109 & 4127 & 29.68 & 32.13 & Classification & 2 & Accuracy \\
 & PROTEINS & 1113 & 39.06 & 72.82 & Classification & 2 & Accuracy \\
 & IMDB-B & 1000 & 19.77 & 96.53 & Classification & 2 & Accuracy \\
 & IMDB-M & 1500 & 13.0 & 65.9 & Classification & 3 & Accuracy \\
 & RDT-B & 2000 & 429.6 & 497.8 & Classification & 2 & Accuracy \\
 & RDT-M & 5000 & 508.5 & 575.5 & Classification & 5 & Accuracy \\
\midrule
% \multirow{8}*{\rotatebox{90}{OGBG} } &  MOLHIV & 41,127 & 25.5 & 27.5 & Binary Classification & 2 & AUROC \\
%  & MOLBACE & 1513 & 34.1 & 36.9 & Binary Classification & 2 & AUROC \\
%  & MOLBBP & 2,039 & 24.1 & 51.9  & Binary Classification & 2 & AUROC \\ 
%  & MOLESOL & 19,717 & 13.3 & 13.7 & Regression &  & RMSE \\
%  & moltox21&  7,831 & 18.6 & 19.3  & Regression & 12  & RMSE\\ 
%  & MOLTOXCAST & 8,576 & 18.8 & 19.3 & Regression & 617  & RMSE \\
%  & MOLPCBA & 437,929 & 26.0 & 28.1  & Regression & 128 & RMSE \\
%  & PPA & 158,100& 243.4 &  2,266.1  & Regression & 1  & RMSE \\
% \midrule
\multicolumn{2}{c}{ZINC-Small} & 12,000 & 23.2 & 24.9 & Regression & 12 & MAE \\
% \multicolumn{2}{c}{ ZINC-FULL} & 249,456 & 23.2 & 49.8 & Regression & 12 & MAE\\
\multicolumn{2}{c}{MOLHIV} & 41,127 & 25.5 & 27.5 & Binary Classification & 2 & AUROC \\
\midrule
\multicolumn{2}{c}{PEPTIDES-FUNC} & 15,535 & 150.9 & 307.3 & Multi-task Classification & 10 & AP \\
\multicolumn{2}{c}{PEPTIDES-STRUCT} & 15,535 & 150.9 & 307.3 & Multi-task Regression & 11 & MAE \\
% Others & & & & & & & & \\
\bottomrule 
\end{tabular}
\begin{tablenotes}
    \item ``\#'' denotes ``number of'', and ``Avg.'' is short for ``Average'', MAE is Mean Absolute Error, AUROC is the Area Under the Curve.
\end{tablenotes}
\end{threeparttable}
\end{table}

\begin{figure}[htbp]
 \centering
 \includegraphics[width=0.95\linewidth]{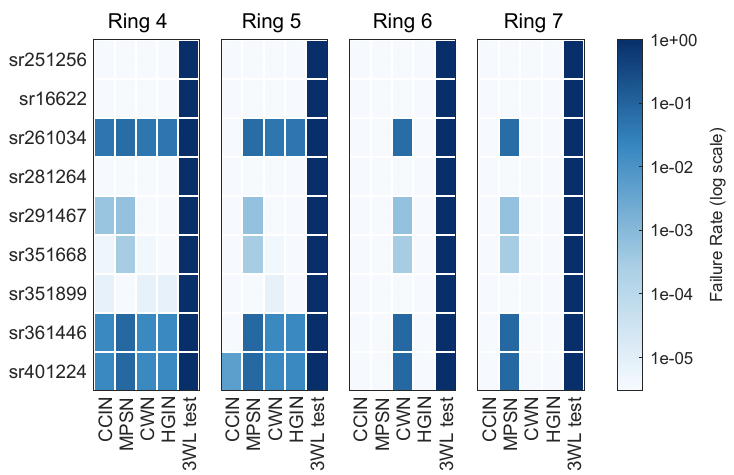} 
\caption{Failure-rate heatmaps on strongly regular graph (SRG) families under different ring sizes ($k=4,5,6,7$). Failure rates are shown on a logarithmic scale; lower values indicate better structural discrimination.} 
 \label{fig03}
\end{figure}
We evaluate CCIN on both synthetic benchmarks and real-world graph learning tasks. Since large-scale benchmarks natively defined on combinatorial complexes remain limited in topological deep learning~\cite{pmlr-v235-papamarkou24a}, we follow lifting-based evaluation protocols~\cite{telyatnikov2024topobenchmarkx,bodnar2021weisfeiler}. Specifically, input graphs are lifted to combinatorial complexes by treating vertices and edges as $0$- and $1$-cells, respectively, and constructing higher-order cells through cyclic lifting. The same lifting protocol is applied to applicable topological baselines for fair comparison. We provide detailed information of datasets as follows.

\begin{table*}[htbp]
% \fontsize{6}{7.2}\selectfont
\scriptsize
\centering
\caption{Graph classification accuracy (Mean$\pm$Std) of our CCIN and the baselines on the datasets from TUDataset collection. Best performance is highlighted in bold. N/A means not available.}
\label{resultstable2_part1}
\begin{minipage}[t]{0.85\textwidth}
\centering
\begin{tabular}{lccccccccc}
\toprule 
 & MUTAG & PTC & PROTEINS & NCI1 & NCI109 & IMDB-B & IMDB-M \\
\midrule
% RWK~\cite{gartner2003graph} & 79.2$\pm$2.1 & 55.8$\pm$0.6 & 59.6$\pm$0.3 & N/A & N/A & N/A & N/A  \\
% GK(k = 3) ~\cite{shervashidze2009efficient} & 81.4$\pm$1.9  & 55.4$\pm$0.4 & 71.4$\pm$1.5 & 62.5$\pm$0.3 & 64.9$\pm$1.0 & 50.9$\pm$3.8 & N/A \\
% % PK & 76.0$\pm$2.7 & 59.5$\pm$2.4 & 73.7$\pm$0.7 & & & & \\
% WL kernel ~\cite{shervashidze2011weisfeiler} & 90.4$\pm$5.7 & 59.9$\pm$4.3 & 75.0$\pm$3.1 & 86.0$\pm$1.8 & 73.8$\pm$3.9 & N/A & N/A \\
% \midrule
DGCNN \cite{zhang2018end}, & 85.8$\pm$1.8 & 58.6$\pm$2.5 & 75.5$\pm$0.9 & 74.4$\pm$0.5 & N/A & 70.0$\pm$0.9 & 47.8$\pm$0.9 \\
IGN ~\cite{cai2022convergence} & 83.9$\pm$13.0 & 58.5$\pm$6.9 & 76.6$\pm$5.5 & 74.3$\pm$2.7 & 72.8$\pm$1.5 & 72.0$\pm$5.5 & 48.7$\pm$3.4 \\
GIN ~\cite{xupowerful} & 89.4$\pm$5.6 & 64.6$\pm$7.0 & 76.2$\pm$2.8 & 82.7$\pm$1.7 & N/A  & 75.1$\pm$5.1 & 52.0$\pm$2.8 \\
PPGNs~\cite{maroninvariant} & 90.6$\pm$8.7 & 66.2$\pm$6.6 & 77.2$\pm$4.7 & 83.2$\pm$1.1 & 80.2$\pm$1.4  & 73.0$\pm$5.8 & 50.5$\pm$3.6 \\
Natural GN ~\cite{de2020natural}& 89.4$\pm$1.6 & 66.8$\pm$1.7 & 71.7$\pm$1.0 & 82.4$\pm$1.3 & N/A & 74.8$\pm$2.0 & 51.3$\pm$1.5 \\
GSN \cite{bouritsas2022improving} & 92.2$\pm$7.5 & 67.2$\pm$7.2 & 75.6$\pm$5.0 & 83.0$\pm$2.0 & N/A  & 73.36 & 51.5  \\
MSPN \cite{bodnar2021weisfeilernips}& 88.3$\pm$10.7 & 62.6$\pm$9.3 & 62.6$\pm$9.3 & 80.4$\pm$1.2 & 79.1$\pm$1.6 & 73.7$\pm$4.0 & \underline{52.1$\pm$3.9} \\
GTR \cite{huang2023growing} & 86.6$\pm$1.4 &65.2$\pm$4.6 & 75.3$\pm$0.8 & N/A &  N/A & 73.1$\pm$0.8 & 79.4$\pm$0.3 \\
CWN \cite{bodnar2021weisfeiler} & 90.0$\pm$7.4 & 62.1$\pm$9.3 & 73.4$\pm$4.4 & \textbf{84.7$\pm$1.7} & \underline{80.3$\pm$1.9} & 73.5$\pm$4.5 & 51.0$\pm$3.1 \\
GraphSNN~\cite{wijesinghe2022new} & 87.3$\pm$3.1 & 61.6$\pm$2.8 & 74.1$\pm$3.2 & 76.4$\pm$1.7 & N/A & 74.8$\pm$3.5 &  N/A \\
RePHINE \cite{immonen2023going} & 87.4$\pm$6.3 & 64.9$\pm$3.7 & 72.3$\pm$1.9 & 80.9$\pm$1.9 & 79.2$\pm$1.7 & 69.4$\pm$3.8 & N/A \\
WLHN \cite{nikolentzos2023weisfeiler} & 86.3$\pm$7.4 & 65.1$\pm$2.4 & 75.9$\pm$1.9 & 79.2$\pm$2.1 &  N/A & 73.4$\pm$3.7 & 49.7$\pm$3.6 \\
G3N \cite{wang2023mathscr} & 89.9$\pm$8.0 & 60.0$\pm$4.8 & \underline{75.9$\pm$2.8} & 78.6$\pm$1.9 & 79.2$\pm$1.3 & 71.0$\pm$2.2 & 45.2$\pm$2.8 \\
HTML \cite{li2024hierarchical}& 88.9$\pm$1.8 & 66.9$\pm$4.2 & 74.9$\pm$0.3 & 78.7$\pm$0.7 & 78.8$\pm$0.6 & 71.7$\pm$0.4 & N/A \\
PathNN \cite{michel2023path} & 90.2$\pm$4.7 & 65.8$\pm$2.7 & 75.2$\pm$3.9 & 77.5$\pm$1.6 & 78.1$\pm$2.1 & 72.6$\pm$3.3 & 50.8$\pm$4.5 \\
TopNets \cite{verma2024topological} & \underline{92.7$\pm$1.9} & 65.7$\pm$3.6 & 73.8$\pm$1.5 & 79.1$\pm$1.2 & 78.4$\pm$0.7 & 73.1$\pm$1.8 & N/A \\
TopoTune \cite{papillontopotune} & 86.4$\pm$6.5 & \underline{67.2$\pm$4.8} & 72.5$\pm$3.1 & 77.6$\pm$1.1 & 77.2$\pm$0.2 & \underline{76.3$\pm$2.7} &  N/A  \\
KGWL \cite{zhangimproved}& 82.5$\pm$5.7 & 66.2$\pm$3.5 & 72.8$\pm$2.6 & 74.9$\pm$3.5 & 78.5$\pm$1.4 & 74.4$\pm$1.9 & 51.6$\pm$3.1 \\
CCIN & \textbf{96.4$\pm$2.1} & \textbf{67.6$\pm$10.1} & \textbf{76.1$\pm$2.5} & \underline{83.2$\pm$1.6} & \textbf{81.1$\pm$2.0} & \textbf{78.3$\pm$4.5} & \textbf{54.7$\pm$3.1} \\
\bottomrule
\end{tabular}
\end{minipage}
\end{table*}

\begin{table*}[htbp]
\centering
\fontsize{6}{7.2}\selectfont
\begin{minipage}[t]{0.36\textwidth}
\centering
% \begin{table}[htbp]
% \small
% \scriptsize 
% \fontsize{6}{7.2}\selectfont
\scriptsize
\setlength{\tabcolsep}{2pt} 
\centering
\caption{Performance on large-scale graph and molecular benchmarks. Results are reported over 9 runs with seed 1-9. (Mean$\pm$Std)}
\label{resultstable2_part2}
\begin{tabular}{lccc}
\toprule
\multirow{2}{*}{Method} & RDT-B & RDT-M & MOLHIV\\
\cmidrule{2-4}
 & \multicolumn{2}{c}{(Accuracy)} & (ROC-AUC) \\
\midrule
WLkernel\cite{shervashidze2011weisfeiler} & 81.0$\pm$3.1 & 52.5$\pm$2.1 & N/A  \\

GIN\cite{xupowerful} & 91.1$\pm$1.8 & 56.2$\pm$1.8 & 77.07$\pm$1.49\\
% RetGK\cite{zhang2018retgk} & 90.8$\pm$0.2 & 54.2$\pm$0.3 & N/A \\
HGCN\cite{chami2019hyperbolic} & 86.3$\pm$1.6 & 52.7$\pm$2.0 & 75.91$\pm$1.48 \\
GSN \cite{bouritsas2022improving} & 91.1$\pm$1.8 & 56.2$\pm$1.8 & 77.99$\pm$1.00 \\
MSPN\cite{bodnar2021weisfeilernips} & 92.7$\pm$0.9 & 57.0$\pm$2.0 & 78.25$\pm$0.31 \\
CWN \cite{bodnar2021weisfeiler} & \underline{93.1$\pm$1.0} & 48.2$\pm$6.6 & 78.58$\pm$0.57
 \\
G3N\cite{wang2023mathscr} & 89.4$\pm$2.1 & N/A & \underline{79.00$\pm$1.34} \\
% ESGNNs & & & 79.09$\pm$0.90 \\
MGNN\cite{kanatsoulis2024counting} & 92.0$\pm$1.8 & 56.1$\pm$1.6 & N/A \\
WLHN \cite{nikolentzos2023weisfeiler} & 90.7$\pm$1.9 & 55.2$\pm$1.2 & 78.41$\pm$0.31 \\
% Subgraphormer\cite{bar2024subgraphormer} & & & 79.48$\pm$1.28
 % \\
GraphSNN ~\cite{wijesinghe2022new} & 92.7$\pm$2.0 & \textbf{57.5$\pm$1.5} & 78.51$\pm$1.72
 \\
HTML\cite{li2024hierarchical} & 90.7$\pm$0.6 & 55.9$\pm$0.4 & 78.68$\pm$0.61 \\
CCIN & \textbf{93.4$\pm$1.1} & \underline{57.3$\pm$1.7} & \textbf{80.45$\pm$1.38} \\
\bottomrule 
\end{tabular}
\end{minipage}
\hfill
\begin{minipage}[t]{0.33\textwidth}
\centering
% \fontsize{6}{7.2}\selectfont
\scriptsize
\setlength{\tabcolsep}{2pt} 
\caption{Performance on peptides benchmarks. Results are reported over 9 runs with seed 1-9. (mean$\pm$std)}
\label{resultstable3}
\begin{tabular}{lcc}
 \toprule 
\multirow{2}{*}{Model} & {PEPTIDES-FUNC} & {PEPTIDES-STRUCT} \\
\cmidrule{2-3}
 & {(AP) $\uparrow$} & {(MAE)} $\downarrow $   \\
\midrule
% GraphTrans & 0.6313$\pm$0.0039 & 0.2777$\pm$0.0025  \\
GIN~\cite{xupowerful} & 0.5498$\pm$0.0079 & 0.3547$\pm$0.0045   \\
GatedGCN~\cite{dwivedigraph} &  0.5498$\pm$0.0079 & 0.3547$\pm$0.0045 \\
% Graph ViT \cite{he2023generalization} & 0.6942$\pm$0.0075 & 0.2449$\pm$0.0016   \\
SAN+LapPE\cite{kreuzer2021rethinking} & 0.6384$\pm$0.0121 & 0.2683$\pm$0.0043 \\
% SAN+EdgeRWSE\cite{dwivedigraph} & 0.6002$\pm$0.0048  & 0.2679$\pm$0.0015   \\
GraphFP~\cite{luong2023fragment} & 0.6267$\pm$0.0073 & 0.3137$\pm$0.0019   \\
2-DRFWL \cite{zhou2023distance} & 0.5953$\pm$0.0048 & 0.2594$\pm$0.0038   \\
GPS~\cite{rampavsek2022recipe} & \underline{0.6435$\pm$0.0041} & 0.2547$\pm$0.0005   \\
% DRGNN & 0.6586$\pm$0.0042 & 0.2495$\pm$0.0015 \\
CWN \cite{bodnar2021weisfeiler} & 0.6237$\pm$0.0038 & 0.2537$\pm$0.0042   \\
GTR \cite{huang2023growing} & 0.6351$\pm$0.0079 & 0.2568$\pm$0.0019 \\
PathNN \cite{michel2023path}& 0.6384$\pm$0.0052 & 0.2540$\pm$0.0046  \\
% GRAND \cite{zhou2023facilitating} & 0.6962$\pm$0.0015 & 0.2867$\pm$0.0009  \\
EMPSN \cite{eijkelboom2023n} & 0.6156$\pm$0.0080 & 0.2539$\pm$0.0015  \\
% B-PEARL\cite{kanatsoulislearning} & 0.6827$\pm$0.0021 & 0.248$\pm$0.00  \\
Subgraphormer\cite{bar2024subgraphormer} & 0.6415$\pm$0.0052 & \underline{0.2529$\pm$0.0020} \\
CCIN & \textbf{0.6493$\pm$0.0262} & \textbf{0.2501$\pm$0.0073}  \\ 
\bottomrule 
\end{tabular}

\end{minipage}
\hfill
\begin{minipage}[t]{0.30\textwidth}
\centering
% \fontsize{6}{7.2}\selectfont
\scriptsize
\setlength{\tabcolsep}{2pt} 
\caption{Mean Absolute Error (Mean$\pm$Std) of different methods on ZINC(12K). Results are reported over 5 runs with seed 1-5. (Mean$\pm$Std)}
\label{table005}
\begin{tabular}{lcc}
 \toprule 
{\multirow{2}{*}{Model}} & \multicolumn{2}{c}{ZINC}  \\
\cmidrule{2-3}
{} & w/o edge feats & w/i edge feats \\
\midrule
GIN ~\cite{xupowerful} & 0.387$\pm$0.015 & 0.252$\pm$0.014 \\
PNA \cite{corso2020principal} & 0.320$\pm$0.032 & 0.188$\pm$0.004 \\
HIMP~\cite{fey2020hierarchical}  & N/A & 0.151$\pm$0.006 \\
% DRFWL \cite{zhou2023distance} & & & 0.152$\pm$0.003 \\
% CWN \cite{bodnar2021weisfeiler} & 0.115$\pm$0.003 & 0.079$\pm$0.006 & N/A \\
% PathNN \cite{michel2023path} & & 0.090$\pm$0.004 & \\
GSN \cite{bouritsas2022improving} & 0.139$\pm$0.007 & 0.115$\pm$0.012  \\
% MSPN & & & \\
MPSN \cite{bodnar2021weisfeiler} & \underline{0.137$\pm$0.008} & \underline{0.094$\pm$0.004}  \\
G3N \cite{wang2023mathscr} & 0.165$\pm$0.018 & 0.128$\pm$0.015 \\
% GD-WL \cite{zhangrethinking} & & 0.086$\pm$0.007 & \\ 
I$^2$-GNN\cite{huangboosting} & N/A & 0.095$\pm$0.007 \\ 
TopoTune \cite{papillontopotune} & 0.247$\pm$0.005 & 0.191$\pm$0.003 \\
MGNN \cite{kanatsoulis2024counting} & 0.140$\pm$0.004 & 0.110$\pm$0.005 \\
CCIN & \textbf{0.125$\pm$0.003} & \textbf{0.082$\pm$0.009}   \\
\bottomrule 
\end{tabular}

\end{minipage}

\end{table*}

\textbf{Evaluation protocols.}
Table~\ref{tab:srdatasets} shows synthetic benchmark consists of strongly regular graph families, where failure rate is used to measure the proportion of non-isomorphic graph pairs not distinguished, e.g., their pairs
cannot be distinguished by the 3-WL test. Table~\ref{tab:realdatasets} displays real-world benchmarks include bioinformatics and social datasets from TUDataset~\cite{Morris+2020}, molecular property prediction datasets such as ZINC and MOLHIV~\cite{wu2018moleculenet,hu2020open}, and long-range peptide benchmarks~\cite{dwivedi2022long,singh2016satpdb}. For classification tasks, we report Accuracy, AUROC, or AP according to the benchmark protocol; for regression tasks, we report MAE. Results are reported as mean and standard deviation.
The experimental details and statistics of 0,1,2-cell after complex lifting construction are in SI Sec.IV.

\subsection{Synthetic Datasets}

\paragraph{Strongly regular graph families.} We evaluate the structural discriminative behavior of CCIN on strongly regular graph (SRG) families. 
% Graphs within the same SRG family share identical graph parameters but may be non-isomorphic, making them useful benchmarks for assessing WL-type discriminative ability. 
For each SRG family, we construct lifted combinatorial complexes with ring sizes $k\in\{4,5,6,7\}$ and measure whether each method can distinguish non-isomorphic graph pairs. Fig~\ref{fig03} reports the failure rate, defined as the proportion of non-isomorphic graph pairs that remain indistinguishable by a method. Lower failure rates indicate stronger discriminative behavior under this evaluation protocol. Compared to baselines, CCIN yields lower failure rates on the evaluated SRG families, especially when larger ring sizes are used. This observation suggests that the lifted higher-order cyclic cells provide additional structural information for distinguishing some challenging regular graph instances.

\subsection{Real-World Datasets}

\paragraph{TUDataset} Table~\ref{resultstable2_part1} reports the graph classification results on representative datasets from TUDataset~\cite{Morris+2020}. Overall, CCIN achieves competitive performance across both bioinformatics and social network benchmarks. Compared with representative message-passing GNNs, CCIN shows favorable results on several datasets. For example, relative to GIN, CCIN improves the accuracy by $7.82\%$ on MUTAG and by approximately $4.30\%$ on IMDB-B. Compared with GraphSNN, CCIN obtains relative gains of $8.90\%$ on NCI1 and $2.70\%$ on PROTEINS. These results suggest that the lifted combinatorial-complex representation can provide useful structural information beyond standard node-level message passing. Compared with topological neural network baselines, CCIN also remains competitive. Relative to CWN, CCIN achieves improvements of $7.11\%$ on MUTAG and $6.53\%$ on IMDB-B. It also shows favorable performance compared with TopNets and TopoTune on several benchmarks. The gains are not uniform across all datasets, which is expected given the diversity of graph sizes and induced higher-order structures. These results indicate that combinatorial-complex-based interactions are valuable for graph classification.

\begin{table*}[htbp]
\centering
\scriptsize
\setlength{\tabcolsep}{3pt}
\caption{Performance comparison of CCIN under different neighborhood-function ablations over 9 runs with seed 1-9. (mean$\pm$std)}
\label{resultstable2}
\begin{tabular}{lccccc|cccc}
\toprule
& MUTAG & PTC & PROTEINS & NCI1 & NCI109 & IMDB-B & IMDB-M & RDT-B & RDT-M \\
\midrule

CCIN-w/o-$\mathcal{B}$
& \underline{90.6$\pm$7.2}
& \underline{62.6$\pm$8.9}
& \underline{75.4$\pm$4.9}
& 79.5$\pm$2.2
& 77.6$\pm$2.1
& 74.7$\pm$4.7
& 51.9$\pm$3.7
& 91.4$\pm$2.4
& 54.3$\pm$2.4 \\

CCIN-w/o-$\mathcal{C}$
& \underline{90.6$\pm$8.3}
& 62.4$\pm$8.5
& 74.1$\pm$3.8
& 79.4$\pm$1.9
& 77.9$\pm$2.0
& 74.3$\pm$4.5
& 52.3$\pm$4.7
& \underline{92.0$\pm$2.4}
& 54.9$\pm$1.6 \\

CCIN-w/o-$\mathcal{N}_{\uparrow}$
& 87.8$\pm$9.2
& 62.1$\pm$6.5
& \underline{75.4$\pm$4.2}
& 73.0$\pm$3.2
& 71.8$\pm$1.4
& 73.9$\pm$2.7
& 52.0$\pm$3.6
& 85.5$\pm$3.2
& 55.0$\pm$3.7 \\

CCIN-w/o-$\mathcal{N}_{\downarrow}$
& 89.8$\pm$7.1
& 61.2$\pm$8.4
& 72.9$\pm$4.2
& 78.4$\pm$1.7
& 77.5$\pm$1.7
& 74.4$\pm$4.9
& \underline{52.8$\pm$3.6}
& 91.7$\pm$1.1
& \underline{55.5$\pm$2.4} \\

CCIN-Full
& 89.1$\pm$7.7
& 61.5$\pm$8.4
& 75.0$\pm$5.4
& \underline{79.6$\pm$1.6}
& \underline{78.5$\pm$1.4}
& \underline{75.7$\pm$4.3}
& 52.0$\pm$3.5
& 89.2$\pm$3.6
& \underline{55.5$\pm$1.5} \\

CCIN
& \textbf{96.4$\pm$2.1}
& \textbf{67.6$\pm$10.1}
& \textbf{76.1$\pm$2.5}
& \textbf{83.2$\pm$1.6}
& \textbf{81.1$\pm$2.0}
& \textbf{78.3$\pm$4.5}
& \textbf{54.7$\pm$3.1}
& \textbf{93.4$\pm$1.1}
& \textbf{57.3$\pm$1.7} \\
\bottomrule
\end{tabular}
\end{table*}

\paragraph{Large-Scale Graph and Molecular Benchmarks.} Table~\ref{resultstable2_part2} summarizes the results on large-scale social and molecular graph benchmarks. On RDT-B, CCIN achieves an accuracy of $93.4$, corresponding to a relative improvement of approximately $2.52\%$ over GIN. On RDT-M, CCIN remains competitive with recent graph and topological methods. For the molecular property prediction benchmark MOLHIV, CCIN obtains a ROC-AUC of $80.45$, giving a relative improvement of approximately $2.38\%$ over CWN. These results suggest that the proposed CCIN can scale to larger graph datasets while retaining competitive predictive performance.

\paragraph{Peptides Benchmarks.} Table~\ref{resultstable3} reports the results on PEPTIDES-FUNC and PEPTIDES-STRUCT, which are designed to evaluate long-range interaction reasoning. On PEPTIDES-FUNC, CCIN achieves an AP of $0.6493$, corresponding to relative improvements of $18.10\%$ over GIN and $4.10\%$ over CWN. It also remains competitive with strong baselines such as GPS and Subgraphormer. On PEPTIDES-STRUCT, CCIN achieves competitive regression performance. These results indicate that combinatorial-complex-based message passing can provide useful structural cues for larger molecular graphs with long-range dependencies.

\paragraph{ZINC-Small.} We further evaluate CCIN on the molecular regression benchmark ZINC-Small, where performance is measured by mean absolute error (MAE), with lower values indicating better performance. As shown in Table~\ref{table005}, CCIN performs favorably compared with classical message-passing models such as GIN and PNA. Without edge features, CCIN reduces MAE relative to GIN and PNA, indicating that lifted higher-order structural information can benefit molecular regression. With edge features, CCIN further improves its performance and remains competitive with advanced cellular and higher-order models. These results support the applicability of CCIN to fine-grained molecular property prediction.

\subsection{Ablation Study }

\begin{figure}[htbp] 
\centering 
\includegraphics[width=1\linewidth]{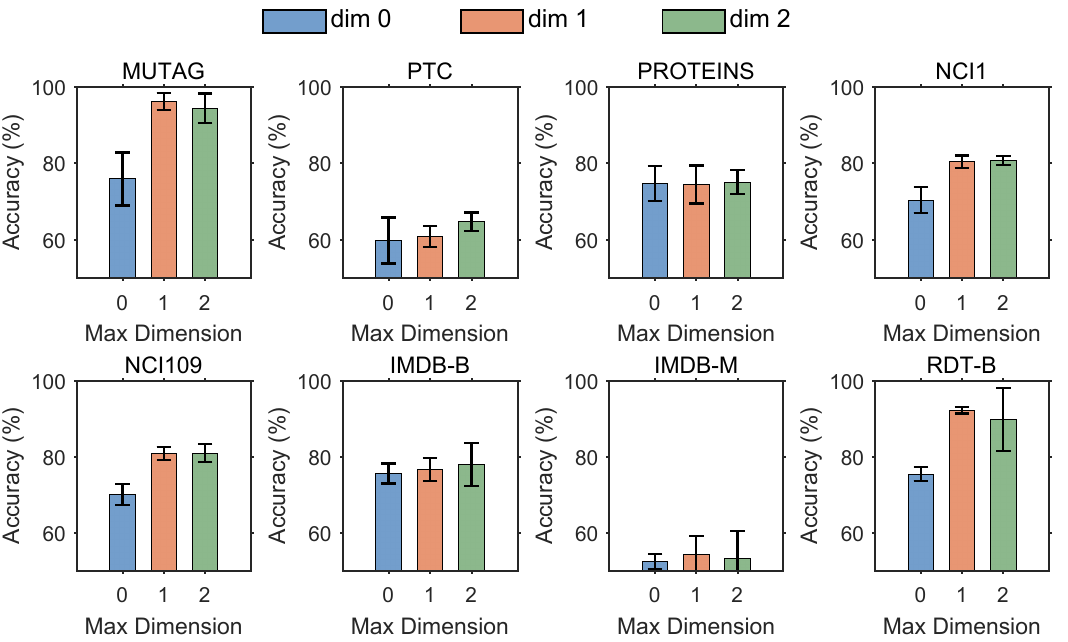} 
\caption{Ablation study of the maximum input dimension in CCIN. Error bars denote standard deviations.} 
\label{fig:dim_ablation} 
\end{figure}

\paragraph{Ablation of Maximum Dimension}
Fig~\ref{fig:dim_ablation} evaluates the effect of the maximum input dimension used in CCIN. The maximum dimension determines the highest-rank cells included in the lifted combinatorial complex and therefore controls the amount of higher-order structural information, the model shows notable performance improvements on multiple datasets. When input with 1-dim structure improves performance by 10.05\%, 10.76\%, and 16.83\% on the NCI1, NCI109, and RDT-B datasets, respectively.  This suggests that edge-level and cycle-induced structures provide useful information for graph classification, especially on molecular and social network datasets. Further CCIN introduces $2$-dimensional cells can bring relatively moderate  improvements. These results indicate that the appropriate introduction of higher-order information can achieve competitive  performance. 
 
% Table~\ref{resultstable2} studies the contribution of different neighborhood channels in CCIN. The results show that the full four-neighborhood aggregation is not always the most effective choice, which is consistent with the redundancy analysis in Theorem~\ref{theorem02}. In some datasets, removing a boundary or co-boundary channel improves performance, suggesting that certain cross-rank messages may introduce redundant or noisy information under specific data distributions. In contrast, the reduced CCIN variant using lower and upper adjacencies achieves competitive results and performs favorably on RDT-B and RDT-M. These observations empirically support the theoretical motivation that lower- and upper-adjacent bridge information can provide a compact and informative message-passing structure under the proposed lifting scheme.

\begin{table*}[htbp]
\centering
\caption{Ablation study on intermediate layer readout function. Initial, Mediate and Final Readout function are initialized as mean, sum and sum functions, respectively.}
\label{tab:readout2}
\scriptsize
\setlength{\tabcolsep}{6pt}
\begin{tabular}{clcccccccc}
\toprule
 &  & MUTAG & PTC  & PROTEINS & NCI1  & NCI109  & IMDB-B  & IMDB-M  & RDT-B \\
\midrule
\multirow{3}{*}{\rotatebox{90}{Initial}}  
& SUM
& 96.3$\pm$5.3 
& 61.8$\pm$6.4 
& 74.2$\pm$2.8 
& \textbf{81.5$\pm$0.9} 
& 79.5$\pm$2.5 
& \textbf{76.7$\pm$2.1} 
& 53.3$\pm$6.8 
& 76.5$\pm$3.9 \\

& MAX 
& 94.4$\pm$2.5 
& 63.7$\pm$6.0 
& 74.2$\pm$3.0 
& 81.4$\pm$1.5 
& 79.4$\pm$2.7 
& 76.0$\pm$1.4 
& 53.3$\pm$3.3 
& 77.2$\pm$3.9 \\

& MEAN 
& \textbf{96.3$\pm$5.2} 
& \textbf{65.7$\pm$5.0} 
& \textbf{75.1$\pm$1.1} 
& 80.9$\pm$1.3 
& \textbf{80.7$\pm$2.5} 
& 76.0$\pm$2.9 
& \textbf{53.6$\pm$5.1} 
& \textbf{78.7$\pm$2.7} \\

\midrule
\multirow{3}{*}{\rotatebox{90}{Mediate}} 
& MEAN
& \textbf{96.3$\pm$2.6}  
& 61.8$\pm$7.2  
& 73.6$\pm$1.9  
& 81.2$\pm$0.6  
& 79.7$\pm$2.5  
& 76.3$\pm$3.1  
& \textbf{53.6$\pm$3.6}  
& \textbf{79.5$\pm$5.3} \\

& SUM 
& 94.4$\pm$7.9  
& 61.8$\pm$2.4  
& 73.9$\pm$5.2  
& \textbf{81.5$\pm$1.4}  
& 79.9$\pm$1.4  
& \textbf{77.3$\pm$2.6}  
& 53.3$\pm$7.1  
& 71.8$\pm$8.3 \\

& MAX 
& 94.4$\pm$7.9  
& \textbf{64.7$\pm$2.4}  
& \textbf{74.5$\pm$3.1}  
& 80.7$\pm$1.2  
& \textbf{80.0$\pm$2.4}  
& 76.0$\pm$5.7  
& 52.4$\pm$5.4  
& 78.7$\pm$2.1 \\

\midrule
\multirow{3}{*}{\rotatebox{90}{Final}} 
& MEAN 
& \textbf{98.2$\pm$2.6} 
& \textbf{62.8$\pm$5.0} 
& 74.2$\pm$2.3 
& \textbf{81.8$\pm$0.1} 
& \textbf{81.1$\pm$1.6} 
& 77.3$\pm$3.7 
& 53.1$\pm$4.2 
& \textbf{79.8$\pm$1.8} \\

& SUM
& 96.3$\pm$2.6 
& 60.8$\pm$3.7 
& 75.1$\pm$3.3 
& 81.4$\pm$1.4 
& 80.8$\pm$0.5 
& \textbf{78.3$\pm$4.0} 
& \textbf{54.3$\pm$5.4} 
& 76.5$\pm$2.6 \\

& MAX
& 94.4$\pm$7.9 
& 61.3$\pm$5.3  
& \textbf{75.7$\pm$2.7} 
& 80.3$\pm$4.4 
& 80.7$\pm$2.1 
& 75.7$\pm$0.5 
& 54.0$\pm$6.3 
& 77.3$\pm$1.7 \\
\bottomrule
\end{tabular}
\end{table*}

\paragraph{Ablation of Neighbor Functions}
Table~\ref{resultstable2} studies the contribution of different neighborhood channels in CCIN. CCIN-Full uses all four neighborhood channels, while CCIN uses only $\mathcal{N}_{\downarrow}$ and $\mathcal{N}_{\uparrow}$.
w/o denotes removing the corresponding channel from CCIN-Full.
The results show that the full four-neighborhood aggregation is not always the most effective choice, which is consistent with the redundancy analysis in Theorem~\ref{theorem02}. For example, on RDT-B, removing $\mathcal{C}$ (w/o-$\mathcal{C}$) achieves an accuracy of $92.0$, corresponding to a relative improvement of approximately $3.14\%$ over CCIN-Full. This suggests that certain neighborhood channels may introduce redundant or noisy information under specific data distributions. In contrast, the reduced CCIN variant using only lower and upper adjacencies achieves the best performance on all evaluated datasets and obtains relative improvements of approximately $8.19\%$, $9.92\%$, and $5.19\%$ over CCIN-Full on MUTAG, PTC and IMDB-M, respectively. These observations empirically support the theoretical motivation that lower- and upper-adjacent bridge information can provide a compact and informative message-passing structure under the proposed lifting scheme.

\paragraph{Ablation of Readout Functions.} 
Table~\ref{tab:readout2} reports the effect of different readout functions at the initial, intermediate, and final stages. The readout functions, including sum, mean, and max, lead to comparable performance in the initial and intermediate stages, although their relative advantages vary across datasets. For example, mean pooling performs favorably on PTC and RDT-B in the initial readout, while max pooling gives competitive results on PTC in the intermediate readout. The final readout has a more visible impact on the overall performance. On MUTAG, for instance, mean readout achieves the highest accuracy among the compared final readout choices. These results indicate that CCIN is relatively robust to the readout choice in earlier stages, whereas the final graph-level pooling function can more directly affect downstream prediction performance.

\subsection{Parameter Sensitivity Study }

\paragraph{Ring Size}
The maximum ring size controls the largest cyclic structure introduced during the lifting process, thereby determining the range of closed higher-order patterns available to CCIN. To evaluate its effect, we vary the ring size from $4$ to $14$ and report classification accuracy on graph benchmarks.The SRG results in Fig~\ref{fig:ring_sensitivity} (a) show a similar trend from the perspective of structural discrimination: increasing the ring size generally reduces the failure rate on several SRG families, but the gain depends on the graph family and the lifting scale.  As shown in Fig~\ref{fig:ring_sensitivity} (b), moderately increasing the ring size can improve performance on several molecular datasets. For example, on NCI1, the accuracy increases from $79.40$ at ring size $4$ to $82.24$ at ring size $7$, corresponding to a relative improvement of $3.58\%$. On NCI109, the best accuracy is obtained at ring size $9$, with a relative improvement of $2.11\%$ over ring size $4$. These results suggest that larger cyclic cells can provide additional closed-structure information for molecular graph classification. The effect is not monotonic across all datasets. On MUTAG, competitive results are already obtained with small or moderate ring sizes, while some intermediate settings lead to lower accuracy. This indicates that excessively enlarging the ring size may introduce redundant or statistically unstable higher-order patterns, especially on small datasets. 
The ring-size study suggests that cyclic lifting is useful for exposing higher-order closed structures, while the optimal scale remains dataset-dependent.

\begin{figure}[h]
    \centering
    \includegraphics[width=1\linewidth]{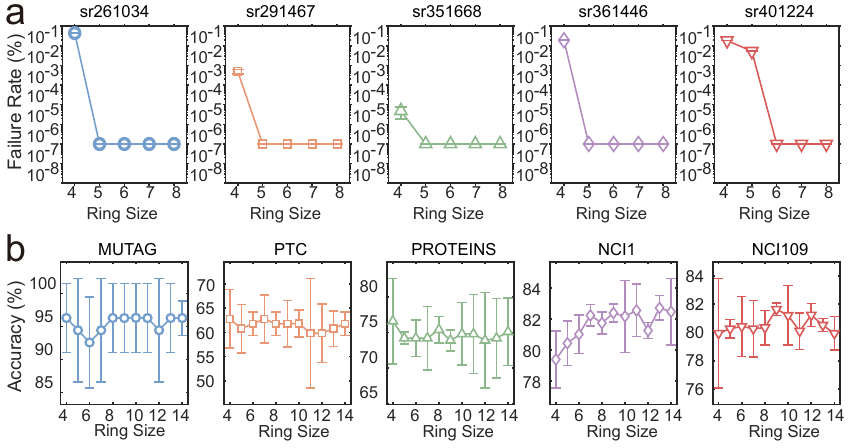}
    \caption{
    Ring-size sensitivity of CCIN. (a) Failure rates(\%) on SRG datasets. (b) Classification accuracy (mean$\pm$std) on graph benchmarks.}
    \label{fig:ring_sensitivity}
\end{figure}

\begin{figure}[h]
    \centering
    \includegraphics[width=1\linewidth]{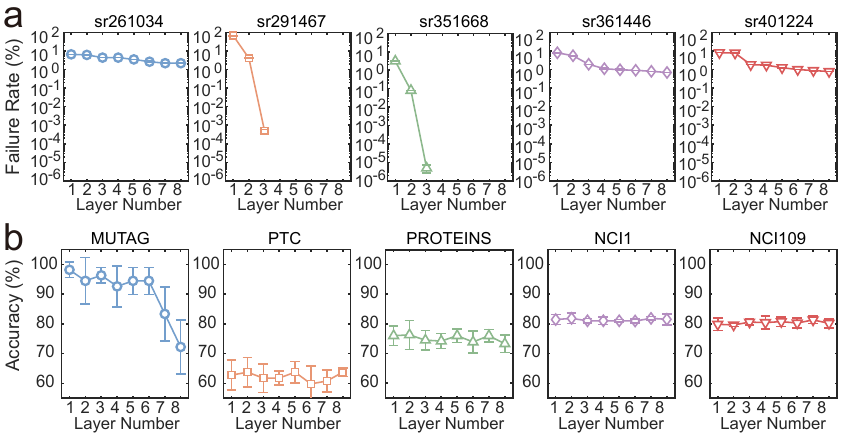}
    \caption{Layer sensitivity of CCIN. (a) Failure rates(\%) on SRG datasets. (b) Classification accuracy (mean$\pm$std) on graph benchmarks.}
    \label{fig:layer_sensitivity}
\end{figure}

\paragraph{Layer Number} We study the effect of network depth by varying the number of CCIN layers from $1$ to $8$. As shown in Fig~\ref{fig:layer_sensitivity}(a), increasing the depth from $1$ to around $3$ or $4$ generally reduces the failure rate on several SRG families, suggesting that moderate depth helps propagate structural information over the lifted combinatorial complex. Further increasing the number of layers brings limited additional reduction, indicating a saturation effect in structural discrimination. 
As reported in Fig~\ref{fig:layer_sensitivity}(b), the classification performance on molecular graph datasets is relatively stable for shallow and moderate depths. On MUTAG, the best performance is achieved with one layer, whereas deeper configurations, especially with $7$ or $8$ layers, lead to a noticeable decrease. This suggests that excessive propagation may introduce over-smoothing or redundant aggregation on small datasets. For larger molecular datasets such as NCI1 and NCI109, the performance varies less significantly across different depths, with $2$--$3$ layers achieving results comparable to deeper configurations. These results reveal a trade-off: shallow models are often sufficient for stable graph classification, while moderate depth is more beneficial for distinguishing challenging synthetic structures.

\begin{figure*}
\includegraphics[width=1\linewidth]{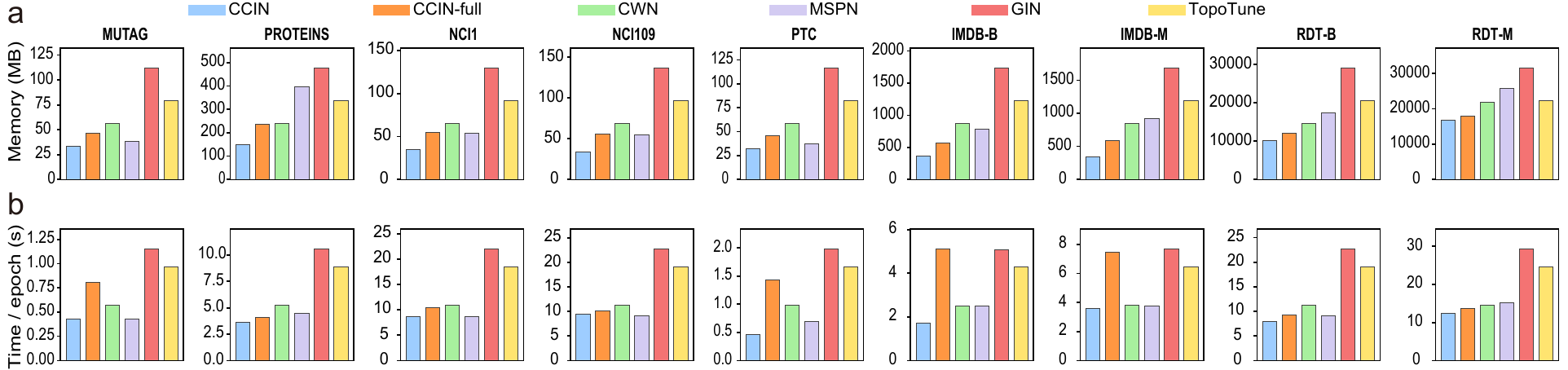}
\caption{Comparison of complexity. (a) Memory consumption (MB) and (b) Runtime (s/epoch). Lower values indicate better resource efficiency.}
\label{fig:resource}
\end{figure*}

\section{Complexity Analysis}
We analyze the computational and memory complexity of the proposed CCWL test and its neural instantiations, CCIN-Full and CCIN. In a combinatorial complex, the structural complexity around a $k$-cell $\sigma \in \mathcal{X}_k$ is characterized by both cross-dimensional and same-dimensional adjacencies. CCIN-Full aggregates messages from all four neighborhoods, including boundary $\mathcal{B}(\sigma)$, co-boundary $\mathcal{C}(\sigma)$, lower adjacency $\mathcal{N}_{\downarrow}(\sigma)$, and upper adjacency $\mathcal{N}_{\uparrow}(\sigma)$. In contrast, CCIN follows the reduced CCWL refinement in Theorem~\ref{theorem02} and only uses the lower- and upper-adjacent bridge channels. Specifically, the boundary size $|\mathcal{B}(\sigma)|$  and the co-boundary size $|\mathcal{C}(\sigma)|$ incur message passing time complexities of $\mathcal{O}(\sum_{\sigma\in \mathcal{X}_k} |\mathcal{B}(\sigma)|)$ and $\mathcal{O}(\sum_{\sigma\in \mathcal{X}_k} |\mathcal{C}(\sigma)|)$, respectively. Conversely, for same-dimensional interactions, the lower neighborhood degree $|\mathcal{N}_{\downarrow}(\sigma)|$ is determined by the number of other $k$-cells sharing a given $(k-1)$-cell $\tau$. By denoting the set of such $k$-cells associated with $\tau$ as $\mathrm{S}_k(\tau)$, the lower adjacency computation yields a time complexity of $\mathcal{O}(\sum_{\tau\in \mathcal{X}_{k-1}} |\mathrm{S}_{k}(\tau)|^2)$. Similarly, the upper neighborhood degree $|\mathcal{N}_{\uparrow}(\sigma)|$ is governed by the $k$-cells sharing a $(k+1)$-cell $\eta$. Denoting this set of $k$-cells as $\mathrm{S}_k(\eta)$, the upper adjacency results in a time complexity of $\mathcal{O}(\sum_{\eta\in \mathcal{X}_{k+1}} |\mathrm{S}_{k}(\eta)|^2)$.
SI Sec.IV reports the time cost required by the lifting construction, together with the number of cells generated in the lifted combinatorial complexes.

Fig~\ref{fig:resource} reports the memory consumption and per-epoch runtime (with lifting time) of different methods across nine graph classification datasets. The results are consistent with the intended complexity reduction of CCIN. In terms of memory usage, CCIN achieves lower memory costs, whereas GIN uses the most memory on most datasets. Compared with CCIN-Full, CCIN reduces memory usage by $29.9\%$ on average. Similar reductions are observed against other higher-order baselines, e.g., CCIN saves more than $40\%$ memory compared with CWN and TopoTune on average, with even larger reductions compared with GIN.
In terms of runtime, CCIN is the fastest or tied for the fastest method on seven out of nine datasets. Compared with CCIN-Full, CCIN decreases per-epoch runtime by $32.5\%$ on average. It also reduces runtime by $62.8\%$ and $55.9\%$ compared with GIN and TopoTune, respectively. These results suggest the practical benefit of the reduced lower/upper-neighborhood formulation: it avoids explicit cross-rank message passing while retaining the bridge information required by the CCWL sufficiency theorem.

\section{Conclusion}

In this paper, we introduced the Combinatorial Complex Weisfeiler--Leman framework for analyzing the expressive power of topological message passing on combinatorial complexes. CCWL defines a unified WL-type refinement over boundary, co-boundary, lower-adjacent, and upper-adjacent neighborhood relations, and establishes connections with several domain-specific WL variants under appropriate lifting maps. We further studied the neighborhood sufficiency problem of CCWL. Under the stated coverage and initialization assumptions, we showed that the reduced refinement using lower and upper adjacencies preserves the distinguishing power of the full four-neighborhood CCWL refinement. This result provides a principled basis for reducing explicit cross-rank propagation while retaining the bridge information required for refinement. Guided by this analysis, we instantiated the reduced rule as the Combinatorial Complex Isomorphism Network.
Experiments on synthetic and real-world benchmarks provide empirical support for the proposed framework. The results show that CCIN can exploit lifted higher-order structures for structural discrimination and achieves competitive performance against representative graph and topological neural network baselines. Ablation and resource-efficiency studies further indicate the practical benefit of the reduced lower/upper-neighborhood design. Future work will study relaxed coverage assumptions, adaptive lifting strategies, and benchmarks naively defined on combinatorial complexes.

\bibliographystyle{IEEEtran} 
\bibliography{New_IEEEtran}

\vfill

\end{document}